\documentclass{article}

\usepackage{microtype}
\usepackage{graphicx}
\usepackage{subfigure}
\usepackage{booktabs} 

\usepackage{hyperref}

\usepackage[accepted]{icml2024}

\usepackage{amsmath}
\usepackage{amssymb}
\usepackage{mathtools}
\usepackage{amsthm,bbm,dsfont}

\usepackage[capitalize,noabbrev]{cleveref}

\theoremstyle{plain}
\newtheorem{theorem}{Theorem}[section]

\newtheorem{lemma}[theorem]{Lemma}

\theoremstyle{definition}
\newtheorem{definition}[theorem]{Definition}
\newtheorem{assumption}[theorem]{Assumption}
\theoremstyle{remark}
\newtheorem{remark}[theorem]{Remark}
\newtheorem{example}[theorem]{Example}

\usepackage[textsize=tiny]{todonotes}

\def\cA{\mathcal{A}}

\def\cD{\mathcal{D}}
\def\cE{\mathcal{E}}

\def\cG{\mathcal{G}}

\def\cM{\mathcal{M}}

\def\cO{\mathcal{O}}

\def\cR{\mathcal{R}}
\def\cS{\mathcal{S}}
\def\cT{\mathcal{T}}

\def\cX{\mathcal{X}}

\def\cZ{\mathcal{Z}}

\newcommand{\rR}{{\mathbb{R}}}
\newcommand{\E}{{\mathbb{E}}}
\newcommand{\Pb}{{\mathbb{P}}}

\newcommand{\reg}{\mathrm{Reg}}

\newcommand{\md}{\mathrm{d}}
\newcommand{\tv}{\mathrm{TV}}

\newcommand{\tO}{\Tilde{O}}

\newcommand{\hcM}{\widehat{\cM}}
\newcommand{\hpi}{\hat{\pi}}
\newcommand{\hM}{\widehat{M}}

\newcommand{\tcO}{\tilde{\mathcal{O}}}

\def\##1\#{\begin{align}#1\end{align}}
\def\$#1\${\begin{align*}#1\end{align*}}

\usepackage{xcolor}         

\usepackage{enumitem}
\def\kl{{\mathrm{KL}}}
\def\ED{{\mathrm{ED}}}
\def\tv{{\mathrm{TV}}}
\def\Var{{\mathrm{Var}}}

\newcommand{\norm}[1]{\left\|#1\right\|}

\newcommand{\argmax}{\mathop{\mathrm{argmax}}}
\newcommand{\IC}{\mathrm{IC}}
\newcommand{\subopt}{\mathrm{SubOpt}}
\newcommand{\Cov}{\mathrm{Cov}}

\icmltitlerunning{Towards Robust Model-Based Reinforcement Learning Against Adversarial Corruption}

\begin{document}

\twocolumn[
\icmltitle{Towards Robust Model-Based Reinforcement Learning \\ Against Adversarial Corruption}



\icmlsetsymbol{equal}{*}

\begin{icmlauthorlist}
\icmlauthor{Chenlu Ye}{equal,yyy}
\icmlauthor{Jiafan He}{equal,zzz}
\icmlauthor{Quanquan Gu}{zzz}
\icmlauthor{Tong Zhang}{aaa}
\end{icmlauthorlist}

\icmlaffiliation{yyy}{The Hong Kong University of Science and Technology.}
\icmlaffiliation{zzz}{University of California, Los Angeles.}
\icmlaffiliation{aaa}{University of Illinois Urbana-Champaign.}

\icmlcorrespondingauthor{Quanquan Gu}{qgu@cs.ucla.edu}
\icmlcorrespondingauthor{Tong Zhang}{tongzhang@tongzhang-ml.org}

\icmlkeywords{Machine Learning, ICML}

\vskip 0.3in
]



\printAffiliationsAndNotice{\icmlEqualContribution} 

\begin{abstract}
This study tackles the challenges of adversarial corruption in model-based reinforcement learning (RL), where the transition dynamics can be corrupted by an adversary. Existing studies on corruption-robust RL mostly focus on the setting of model-free RL, where robust least-square regression is often employed for value function estimation. However, the uncertainty weighting techniques cannot be directly applied to model-based RL. In this paper, we focus on model-based RL and take the maximum likelihood estimation (MLE) approach to learn transition model. Our work encompasses both online and offline settings. In the online setting, we introduce an algorithm called corruption-robust optimistic MLE (CR-OMLE), which leverages total-variation (TV)-based information ratios as uncertainty weights for MLE. We prove that CR-OMLE achieves a regret of $\tilde{\mathcal{O}}(\sqrt{T} + C)$, where $C$ denotes the cumulative corruption level after $T$ episodes. We also prove a lower bound to show that the additive dependence on $C$ is optimal. We extend our weighting technique to the offline setting, and propose an algorithm named corruption-robust pessimistic MLE (CR-PMLE). Under a uniform coverage condition, CR-PMLE exhibits suboptimality worsened by $\mathcal{O}(C/n)$, nearly matching the lower bound. To the best of our knowledge, this is the first work on corruption-robust model-based RL algorithms with provable guarantees.
\end{abstract}

\section{Introduction}
Reinforcement learning (RL) seeks to find the optimal policy within an unknown environment associated with rewards and transition dynamics. A representative model for RL is the Markov decision process (MDP) \citep{sutton2018reinforcement}. While numerous studies assume static rewards and transitions, the environments in real-world scenarios are often non-stationary and vulnerable to adversarial corruption. For instance, autonomous vehicles frequently fall victim to misled navigation caused by hacked maps and adversarially contaminated traffic signs \citep{eykholt2018robust}. Similarly, in smart healthcare systems, an adversary with partial knowledge can easily manipulate patient statuses \citep{nanayakkara2022unifying}. Under this situation, standard RL algorithms often fail to find policies robust to such adversarial corruption. Therefore, how to identify the optimal policies against adversarial corruption has witnessed a flurry of recent investigations. 

In this work, we focus on the scenario where the adversary can manipulate the transitions before the agent observes the next state. Achieving a sub-linear regret bound under transition corruption has been shown to be computationally challenging with full information feedback \citep{abbasi2013online}, and even information-theoretically challenging with bandit feedback \citep{tian2021online}. Therefore, a series of studies have introduced constraints on the level of corruption, such as limiting the fraction of corrupted samples \citep{zhang2022corruption} or the cumulative sum of corruptions over $T$ rounds \citep{lykouris2018stochastic,gupta2019better,he2022nearly,ye2023corruptiona,ye2023corruptionb,yang2023towards}. While these works exhibit a sub-linear regret bound, to the best of our knowledge, existing works all focus on the setting of model-free RL, where the agent directly learns a policy or a value function from the experiences gained through interactions.

In contrast to model-free RL, 
in model-based RL, the agent learns an explicit model of the environment and utilizes this model for decision-making. This paradigm not only exempts from Bellman completeness \citep{jin2021bellman} (which is a standard assumption for model-free RL) but also has demonstrated impressive sample efficiency in both theories \citep{kearns2002near,
brafman2002r,auer2008near,sun2019model,agarwal2022model} and applications \citep{chua2018deep,nagabandi2020deep,schrittwieser2020mastering}, such as robotics \citep{polydoros2017survey} and autonomous driving \citep{wu2021reinforcement}. In the online setting, one of the most representative frameworks is Optimistic Maximum Likelihood (OMLE) \citep{liu2023optimistic}. This framework establishes a confidence set based on log-likelihood and selects the most optimistic model within the confidence set. However, how to make model-based RL provably robust against adversarial corruption remains an open problem.


In this paper, we resolve this open problem in both online and offline settings with a general function approximation. 
For simplicity, we assume that the reward is known and mainly focuses on learning the transition, which is both unknown and subject to corruption. In particular, we first introduce a cumulative measure for the corruption level of the transition probabilities.

In the online setting, to enhance the resilience of exploitation and exploration to potential corruption, we integrate OMLE with a novel uncertainty weighting technique. Distinct from previous works \citep{he2022nearly,ye2023corruptiona,ye2023corruptionb} in the bandits or model-free RL that characterize uncertainty with rewards or value functions, we characterize the uncertainty with the probability measure of transitions. More specifically, we define the uncertainty as a total-variation (TV)-based information ratio (IR) between the current sample and historical samples. Notably, the introduced IR resembles the eluder coefficient in \citet{zhang2023mathematical}. The samples with higher uncertainty are down-weighted since they are more vulnerable to corruption. Additionally, similar to \citet{liu2023optimistic}, we quantify the complexity of the model class $\cM$ with $TV$-based eluder dimension and establish a connection between the TV-based eluder dimension and the cumulative TV-based IR following \citet{ye2023corruptiona} (by considering uncertainty weights).

In addition to the online setting, we further apply the uncertainty weighting technique to the offline learning scenario, and employ pessimistic MLE with uncertainty weights. 

\paragraph{Contributions.} 
We summarize our contributions as follows.
\begin{itemize}[leftmargin=*]
\setlength\itemsep{-0.5em}
    \item For the online setting, we propose an algorithm CR-OMLE (Corruption-Robust Optimistic MLE) by integrating uncertainty weights with OMLE.
    This algorithm enjoys a regret bound of $\tcO(H\log B\sqrt{T\log|\cM|\ED} + CH\cdot\ED)$, where $H$ is the episode length, $B$ is an upper bound of transition ratios, $T$ is the number of episodes, $|\cM|$ is the cardinality of the model class, $\ED$ is the eluder dimension, and $C$ is the corruption level over $T$ episodes. Moreover, we construct a novel lower bound of $\Omega(HdC)$ for linear MDPs \citep{jin2020provably} with dimension $d$. As a result, the corruption-dependent term in the upper bound matches the lower bound, when reduced to the linear setting. 
    These results collectively suggest that our algorithm is not only robust to adversarial corruption but also near-optimal with respect to the corruption level $C$. 
    \item For the offline setting, we propose an algorithm CR-PMLE (Corruption-Robust Pessimistic MLE) and demonstrate that,  given a corrupted dataset with an uniform coverage condition, the suboptimality of the policy generated by this algorithm against the optimal policy can be upper bounded by $\tcO(H\log B\sqrt{\log |\cM|/(T\Cov(\cM,\cD))} + CH/(T\Cov(\cM,\cD)))$, where $T$ is the number of trajectories, $\Cov(\cM,\cD)$ is the coverage coefficient with respect to model class $\cM$ and dataset $\cD$. 
Furthermore, we establish a lower bound of $\Omega \big({C}/{(T \cdot \Cov(\cM,\cD)})\big)$ for the suboptimality within the linear MDP region. Remarkably, the corruption-dependent term in the upper bound aligns closely with the lower bound 
up to $H$. It is also worth noting that this is also the first provable model-based offline RL algorithm even without considering the corruption.
\end{itemize}

\paragraph{Notation.} For two probabilities $P_1, P_2$, we denote the total variation and the Hellinger distance of $P_1$ and $P_2$ by $\tv(P_1\|P_2)=\|P_1-P_2\|_1/2$ and $H(P_1\|P_2)^2=\E_{z\sim P_1}(\sqrt{\md P_2/\md P_1}-1)^2$. We use the short-hand notation $z=(x,a)$ for $(x,a)\in\cX\times\cA$ when there is no confusion. For two positive sequences $\{f(n)\}_{n=1}^\infty$, $\{g(n)\}_{n=1}^\infty$,
we say $f(n)=\cO(g(n))$ if there exists a constant $C > 0$ such that $f(n)\le Cg(n)$ for all $n \ge 1$, and $f(n)=\Omega(g(n))$ if there exists a constant $C > 0$ such that $f(n)\ge Cg(n)$ for all $n \ge 1$. We use $\tcO(\cdot)$ to omit polylogarithmic factors.

\section{Related Work}
\textbf{Corruption-Robust Bandits and RL.} The adversarial corruption is first brought into multi-armed bandit problems by \citet{lykouris2018stochastic}, where rewards face attacks by $c_t$ in each round, with the cumulative corruption level over $T$ rounds represented as $C=\sum_{t=1}^T|c_t|$. A regret lower bound of $o(T)+O(C)$ was constructed by \citet{gupta2019better}, illustrating that an ``entangled" relationship between $T$ and $C$ is ideal. 
Subsequently, the corruption was extended to linear contextual bandits by \citet{bogunovic2021stochastic,ding2022robust,foster2020adapting,lee2021achieving,zhao2021linear,kang2023robust}. However, these approaches either obtain sub-optimal regret bounds or necessitate additional assumptions. \citet{he2022nearly} first achieved a regret bound that matches the lower bound by utilizing an uncertainty weighting technique. Turning to MDPs, the corruption in transitions triggers considerable interest. \citet{wu2021reinforcement} first studied corruption on rewards and corruption simultaneously for tabular MPDs. Later, a unified framework to deal with unknown corruption was established by \citet{wei2022model}, involving a weak adversary who decides the corruption amount for each action before observing the agent's action. Subsequently, \citet{ye2023corruptiona,ye2023corruptionb} extended the uncertainty weighting technique from \citet{he2022nearly} to RL with general function approximation for both online and offline settings respectively, aligning with the lower bound in terms of the corruption-related term. However, previous works all focused on the model-free setting, where models are learned via regression. Notably, corruption in the model-based setting, where the models are learned via maximum likelihood estimation (MLE), remains an unexplored area in the literature.

\textbf{Model-Based RL.} There is an emerging body of literature addressing model-based RL problems, from tabular MDPs \citep{kearns2002near,
brafman2002r,auer2008near} to rich observation spaces with linear function approximation \citep{yang2020reinforcement,
ayoub2020model}. Extending beyond linear settings, recent studies \citep{sun2019model,agarwal2020flambe,uehara2021representation,du2021bilinear,agarwal2022model,wang2023benefits,liu2023optimistic,foster2021statistical,zhong2022gec,chen2022general} have explored general function approximations. Simultaneously, a parallel line of research has focused on the model-free approach with general function approximations \citep{jin2021bellman,
jiang2017contextual,chen2022general,zhang2023mathematical,agarwal2023vo,zhao2023nearly,di2023pessimistic,liu2023one}. Among the works, the most relevant one to ours is \citet{liu2023optimistic}, who makes optimistic estimations via a log-likelihood approach and generalizes an eluder-type condition from \citet{russo2013eluder}. Despite numerous works on online model-based RL problems, to the best of our knowledge, there is still a gap in the literature regarding a theoretical guarantee for offline model-based RL. 

\textbf{Distributional Robust RL.} Our work shares the same goal of robustness with another line of work, called distributional robust RL, which formulates the uncertainty of the transitions as an uncertainty set \citep{roy2017reinforcement,badrinath2021robust,wang2021online} or a set of distributions surrounding the nominal models \citep{zhou2021finite,shi2024curious,clavier2023towards}. In this setting, there is no corruption during the training process or exploration phase, and the aim is to find a robust policy that maintains a near-optimal policy when there exists a small distributional shift between the training and real environments. In comparison, our work focuses on cases where the training data is corrupted, and we aim to identify the optimal policy for the hidden environment even with a few corrupted observations. Hence, these two categories of works study different notions of robustness and have different challenges.

\section{Preliminaries}

Consider an episodic MDP$(\cX,\cA,H,\Pb,r)$ specified by a state space $\cX$, an action space $\cA$, the number of transition steps $H$, a group of transition models $P=\{P^h:\cX\times\cA\rightarrow\Delta(\cX)\}_{h=1}^H$ and a reward function $r=\{r^h:\cX\times\cA\rightarrow\rR\}_{h=1}^H$. Given a policy $\pi=\{\pi^h:\cX\rightarrow\Delta(\cA)\}_{h=1}^H$, we define the $Q$-value and $V$-value functions as the cumulative sum of rewards from the $h$-th step for policy $\pi$: $Q^h_{\pi}(x,a)=\sum_{h'=h}^H \E^*_{\pi}[r^{h'}(x^{h'},a^{h'})\,|\, x^h=x,a^h=a]$, $V^h_{\pi}(x)=\sum_{h'=h}^H \E^*_{\pi}[r^{h'}(x^{h'},a^{h'})\,|\,x^h=x]$. For simplicity, we assume that the reward function $r$ is known and is normalized: $\sum_{h=1}^Hr^h(x^h,a^h)\in[0,1]$. Hence, we can delve into learning the transition model by interacting with the system online or using offline data. 

In \textit{model-based} RL, we consider an MDP model class $\cM$. Without loss of generality, we assume that the class $\cM$ has finite elements. Note that the finite class assumption is only to simplify the analysis. We can extend the proof to an infinite model class by taking a finite covering. Each model $M\in\cM$ depicts transition probability $P_M^h(x^{h+1}|x,a)$. We use $\E^M[\cdot]$ to represent the expectation over the trajectory under transition probability $P_M^{\pi}$, $V_M^{\pi}$ and $Q_M^{\pi}$ to denote the $V$-value and $Q$-value function of model $M$ and policy $\pi$, and $\pi_M=\pi_{Q_M}$ to denote the optimal policy for the model $M$. Then, we short-notate $V_M^{\pi_M,h}$ and $Q_M^{\pi_M,h}$ as $V_M^h$ and $Q_M^h$. Additionally, we suppose that there exists an underlying true model $M^*\in\cM$ such that the true transition probability is $P_{M_*}^h(x^{h+1}|x^h,a^h)$, and use the short-hand notation $P_*^h=P_{M^*}^h$ and $\E^*[\cdot]=\E^{M_*}[\cdot]$. Given a model $M$, the model-based Bellman error is defined as
\$
&\cE^h(M,x^h,a^h)\\
&= Q_M^h(x^h,a^h) - R_*^h(x^h,a^h) - \E^*[V_M^{h+1}(x^{h+1})|x^h,a^h]\\
&= \E^M[V_M^{h+1}(x^{h+1})|x^h,a^h]-\E^*[V_M^{h+1}(x^{h+1})|x^h,a^h].
\$
For analysis, we assume that for any model $M\in\cM$, the ratio between transition dynamic $P_M$ and $P_*$ is upper-bounded.
\begin{assumption}\label{as:Bounded Condition}
There exists a constant $B>0$ such that
{\small
    \$
        \sup_{h\in[H],M\in\cM}\max\bigg\{\Big\|\frac{P_*^h}{P_M^h}\Big\|_{\infty}\, \Big\|\frac{P_M^h}{P_*^h}\Big\|_{\infty}\bigg\} \le B.
    \$}
\end{assumption}

\paragraph{Online Learning.}
In online learning, an agent interacts with the environment iteratively for $T$ rounds with the \textbf{goal} of learning a sequence of policy $\{\pi_t\}_{t=1}^T$ that minimize the cumulative suboptimality: $
\reg(T)= \sum_{t=1}^T\left[V_*^1(x_t^1)-V_{\pi_t}^1(x_t^1)\right].
$

To depict the structure of the model space $\cM$, we introduce the model-based version of the eluder dimension similar to \citet{russo2013eluder,liu2023optimistic}.

\begin{definition}[$\epsilon$-Dependence]
For an $\epsilon>0$, we say that a point $z$ is $\epsilon$-dependent on a set $\mathcal{Z}$  with respect to $\mathcal{M}$ if there exists $M,M'\in\mathcal{M}$ such that $\sum_{z'\in\mathcal Z}|\mathrm{tv}(P^h_M(\cdot|z')\|P^h_{M'}(\cdot|z'))|^2 \le\epsilon^2$ implies $|\mathrm{tv}(P^h_{M}(\cdot|z)\|P^h_{M'}(\cdot|z))|\le\epsilon$.
\end{definition}
This dependence means that a small in-sample error leads to a small out-of-sample error. Accordingly, we say that $z$  is $\epsilon$-independent of $\cZ$ if it is not $\epsilon$-dependent of $\cZ$.

\begin{definition}[$\tv$-Eluder Dimension]\label{df:tv-Eluder Dimension}
    For a model class $\cM$, an $\epsilon>0$, the $\tv$-eluder dimension $\ED(\cM,\epsilon)$ is the length $n$ of the longest sequence $\{z_1,\ldots,z_n\}\subset\cX\times\cA$ such that for some $\epsilon'\ge\epsilon$ and all $h\in[H]$, $i\in[n]$, $z_i$ is $\epsilon'$ independent of its predecessors.
\end{definition}
Particularly, \citet{liu2023optimistic} also defines the TV-eluder condition. The distinction lies in their consideration of a sequence of policies, while we focus on a sequence of state-action samples. In Theorems \ref{th:TV-Eluder Dimension for Tabular MDPs} and \ref{th:TV-Eluder Dimension for Linear MDPs}, we offer examples of tabular and linear MDPs where the TV-eluder dimension can be bounded. 
To facilitate analysis, we formulate the TV-norm version of the eluder coefficient drawing inspiration from the approach of \citet{zhang2023mathematical}. Given a model class $\cM'$, sample set $\cS_t^h=\{z_s^h\}_{s=1}^{t}$ and some estimator $\bar M_t$ from $\cS_t^h$ (will be specified by later algorithms), the information ratio (IR) between the estimation error and the training error within $\cM'$ with respect to $\bar M_t$ is
{\small
\#\label{eq:tv_information_ratio}
    I^h(\lambda,\cM',\cS_t^h)= \min\bigg\{1,\sup_{M\in\cM'}\frac{l_t^h(M,\bar M_t)}{\sqrt{\lambda+\sum_{s=1}^{t-1}l_s^h(M,\bar M_t)^2}}\bigg\},
\#}
where $l_t^h(M,M')$ denotes $\tv(P_{M}^h(\cdot|z_t^h)\|P_{M'}^h(\cdot|z_t^h))$.
We use the linear MDP as an example to illustrate IR. If the transition can be embedded as $P^h_M(x^{h+1}|z^h)=\nu^h(M,x^{h+1})^{\top}\phi^h(z^h)$, according to Example \ref{eg:Information Ratio for Linear Model}, the IR can be reduced to
\$
\min\Big\{1,\|\phi^h(z^h)\|_{(\Sigma_t^h)^{-1}}\Big\},
\$
where $\Sigma_t^h = \sum_{s=1}^{t-1}\phi^h(z_s^h)\phi^h(z_s^h)^{\top}$. Intuitively, this quantity represents the uncertainty of vector $\nu^h(M,x^{h+1})$ in the direction of $\phi(z_t^h)$.
After observing the state-action pair $(x_t^h,a_t^h)$ in each round, an adversary corrupts the transition dynamics from $P_*^h$ to $P_t^h$ and the agent receives the next state $x_t^{h+1}$ induced by $P_t^h(\cdot|x_t^h,a_t^h)$. To measure the corruption level, we define cumulative corruption.
\begin{definition}[Corruption Level]\label{def:Corruption Level}
We define the corruption level $C$ as the minimum value that satisfies the following property: For the sequence $\{x_t^h,a_t^h\}_{t,h=1}^{T,H}$ chosen by the agent and each stage $h\in[H]$,
    \$
    &c_t^h = c_t^h(x_t^h,a_t^h) = \sup_{x^{h+1}\in\Delta_t(\cX)}\Big|\frac{P_t^h(x^{h+1}|x_t^h,a_t^h)}{P_*(x^{h+1}|x_t^h,a_t^h)}-1\Big|,\\
    &\sum_{t=1}^Tc_t^h \le C,
    \$
    where $\Delta_t(\cX)$ is the support of $P_*^h(\cdot|x_t^h,a_t^h)$.
\end{definition}
We use $\E^t$ to denote expectations evaluated in the corrupted transition probability $P_t$. Note that $P_t^h(\cdot|x_t^h,a_t^h)$ and $P_*^h(\cdot|x_t^h,a_t^h)$ have the same support, i.e., the adversary cannot make the agent transfer to a state that is impossible to be visited under $P_*$.

\paragraph{Offline Setting.} 
For offline environment, the agent independently collects $T$ trajectories $\cD={(x_t^h,a_t^h,r_t^h)}_{i,h=1}^{T,H}$ during the data collection process. For each trajectory $t\in[T]$, an adversary corrupts the transition probability to $P_t^h(\cdot|x_t^h,a_t^h)$ after observing $(x_t^h,a_t^h)$. The corruption level $C$ is measured in the same way as in Definition \ref{def:Corruption Level}. The only difference is that corruption occurs during the collection process. The \textbf{goal} is to learn a policy $\hat{\pi}$ such that the suboptimality with respect to the uncorrupted transition is sufficiently small:
\[
\subopt(\hat{\pi},x^1) = V_*^h(x^1) - V_{\hat{\pi}}^h(x^1).
\]

\section{Online Model-based RL}\label{sec:online}

In this section, we will first present a novel uncertainty weighting technique tailored for model-based RL, followed by the corruption-robust model-based online RL algorithm and its analysis. 

\subsection{Uncertainty Weighting for Model-based RL}\label{ss:Uncertainty Weighting for Model-based RL}
In this subsection, we will illustrate how our uncertainty-weighting technique differs from the one for model-free RL. The main difficulty arises from the different estimators applied by these two settings. In model-free RL, we estimate the value functions using a value target regression. In contrast, for model-based RL, we directly estimate the hidden transition probability with MLE.

In detail, for model-free RL, let $f^*$ be the uncorrupted true value function and $\hat{f}$ be the least squares estimator with weights $\sigma_s$. The corruption term can be directly decomposed as the multiplication between uncertainty $U_s = \sup_{f\in\mathcal{F}}\frac{|\hat{f}(z_s) - f(z_s)|}{\sqrt{\lambda + \sum_{j=1}^t(\hat{f}(z_j)-f(z_j))^2/\sigma_j}}$ and corruption \citep{ye2023corruptiona,ye2023corruptionb}:
{\footnotesize
\$
&\sum_{s=1}^t\frac{(\hat{f}(z_s) - f^*(z_s))^2}{\sigma_s} = \sum_{s=1}^t\frac{(\hat{f}(z_s) - y_s)^2 - (f^*(z_s) - y_s)^2}{\sigma_s}\\
&\qquad + 2\sum_{s=1}^t\frac{(\hat{f}(z_s) - f^*(z_s))\epsilon_s}{\sigma_s} + 2\underbrace{\sum_{s=1}^t\frac{U_s c_s}{\sigma_s}\cdot\beta}_{\text{Corruption~ term}}.
\$}
By choosing $\sigma_s\ge U_s/\alpha$, we can convert the corruption term to $\alpha C \beta$. To conclude, only if we transform the corruption term into the multiplication between uncertainty $U_s$ and corruption $c_s$, can the uncertainty-related weights cancel $U_s$ and bring in a small hyper-parameter $\alpha$.

While for model-based RL, it is difficult to decompose the corruption term as the multiplication between uncertainty and corruption terms from the log-likelihood, especially when uncertainty is based on the total variation (TV) norm. Specifically, let $P_M$ be the estimated probability, $P_*$ be the true probability, and $P_t$ be the corrupted probability at round $t$. Neglecting the upscript $h$ for convenience, we write:
{\footnotesize
\$
&\mathbb{E}\frac{1}{\sigma_s}\log\sqrt{\frac{\mathrm{d}P_{\widehat M}(x|z_s)}{\mathrm{d}P_*(x|z_s)}} = \underbrace{\frac{1}{\sigma_s}\int \mathrm{d}P_*(x|z_s)\log\sqrt{\frac{\mathrm{d}P_{\widehat M}(x|z_s)}{\mathrm{d}P_*(x|z_s)}}}_{\text{Uncorrupted~term}}\\
&\qquad + \underbrace{\frac{1}{\sigma_s}\int (\mathrm{d}P_s(x|z_s) - \mathrm{d}P_*(x|z_s)) \log\sqrt{\frac{\mathrm{d}P_{\widehat M}(x|z_s)}{\mathrm{d}P_*(x|z_s)}}}_{\text{Corrupted~term}}.
\$

To make the corrupted term a multiplication between uncertainty and corruption, we use $\log x \leq x-1$ and decompose the integration into two regions according to whether $\mathrm{d}P_s - \mathrm{d}P_*$ is positive or not. Then, we deal with the variance similarly and use Assumption \ref{as:Bounded Condition} and the Freedman inequality to obtain
{\small
\$
&\E\sum_{s=1}^{t-1}\frac{1}{\sigma_s^h}\log\sqrt{\frac{\md P^h_{\bar M_t}(x^{h+1}|z_s^h)}{\md P^h_*(x^{h+1}|z_s^h)}}\\
&\lesssim -\sum_{s=1}^{t-1}\frac{\tv\big(P^h_*(\cdot|z_s^h)\big\|P^h_{\bar M_t}(\cdot|z_s^h)\big)^2}{\sigma_s^h} + \underbrace{\sum_{s=1}^{t-1}\frac{c_s^h U_s(z_s^h)\beta}{\sigma_s^h}}_{\text{Corruption~Term}},
\$}
where we use $\lesssim$ to omit constants for conciseness, and $U_s$ is the uncertainty defined in \eqref{eq:uncertainty}. Therefore, our new technique makes it possible for an TV-based uncertainty weight to control the corruption term. 
}

\subsection{Algorithm: CR-OMLE}
\begin{algorithm}[th]
\small
    \caption{Corruption-Robust Optimistic MLE (CR-OMLE)}
    \label{alg:CR-OMLE}
    \begin{algorithmic}[1]
        \STATE \textbf{Input:} $\cM_1=\cM$ and $\cD=\{\}$.
        \FOR{t=1,\ldots,T}
            \STATE Observe $x_t^1$;
            \STATE Construct $M_t=\argmax_{M\in\cM_t}V_M^1(x_t^1)$;
            \STATE Let $\pi_t$ be the greedy policy of $V_{M_t}^1$;
            \STATE Collect new trajectory $\{x_t^1,a_t^1,r_t^1,\ldots,x_t^H,a_t^H,r_t^H\}$ and update it into $\cD$;
            \STATE Set $\sigma_t^h$ as \eqref{eq:unc_weight}, and calculate 
            \$      \bar{M}_{t+1}=\argmax_{M'\in\cM_t}\sum_{s=1}^{t}\sum_{h=1}^H \frac{\log P_{M'}^h(x_s^{h+1}|x_s^h,a_s^h)}{\sigma_s^h};
            \$
            \STATE Find $\beta$ and construct the confidence set $\cM_{t+1}$ as
            \#\label{eq:confidence_set}
            \Big\{&M\in\cM_t: \forall~h\in[H], ~\sum_{s=1}^t \frac{\log P_M^h(x_s^{h+1}|x_s^h,a_s^h)}{\sigma_s^h}\notag\\ 
            &\qquad\ge \sum_{s=1}^t \frac{\log P_{\bar{M}_{t+1}}^h(x_s^{h+1}|x_s^h,a_s^h)}{\sigma_s^h} - \beta^2\Big\}.
            \#
        \ENDFOR
    \end{algorithmic}
\end{algorithm}

We propose Algorithm \ref{alg:CR-OMLE} for the online episodic learning setting. In each round $t$, after observing an initial state $x_t^1$, the agent follows the principle of optimism and selects the model $M_t$ from the confidence set $\cM_t$ to maximize the value function $V_M^1(x_t^1)$. Subsequently, the agent follows the greedy policy $\pi_t$ to collect a trajectory $\{(x_t^h, a_t^h, r_t^h)\}_{h=1}^H$. In constructing the confidence set, the agent initially learns $\bar M_{t+1}$ by maximizing the weighted log-likelihood, where the weight $\sigma_t^h$ is a truncated variant of the Information Ratio (IR) \eqref{eq:tv_information_ratio}, referred to as uncertainty
\#\label{eq:unc_weight}
    \sigma_t^h = \max\Big\{1, \frac{1}{\alpha}U_t(x_t^h,a_t^h)\Big\},
\#
where the hyper-parameter $\alpha>0$ is inversely proportional to the corruption level, and we define uncertainty as the Information Ratio (IR) with weights:
{\small
\#\label{eq:uncertainty}
&U_t(x_t^h,a_t^h)\notag\\
& = \sup_{M\in\cM_t}\frac{\tv\big(P_{M}^h(\cdot|x_t^h,a_t^h)\|P_{\bar M_t}^h(\cdot|x_t^h,a_t^h)\big)}{\sqrt{\lambda+\sum_{s=1}^{t-1}\tv\big(P_{M}^h(\cdot|x_s^h,a_s^h)\|P_{\bar M_t}^h(\cdot|x_s^h,a_s^h)\big)^2/\sigma_{s}^h}}.
\#}
Subsequently, the confidence set is a subset of $\cM_t$ that introduces a $\beta^2$-relaxation to the maximum of the log-likelihood as shown in \eqref{eq:confidence_set}.

\textbf{Computation Efficiency.}
If the model class is finite or has a finite cover with cardinality $M$, the computational complexity is $MTH$ since in Line 4 and 7 of Algorithm \ref{alg:CR-OMLE}, we have to search over also the components in the model class. If the model class is complicated, we have to acknowledge that methods like version space methods \citep{liu2023optimistic,jiang2017contextual,wang2023benefits,jin2021bellman}, which construct a confidence set, face computational drawbacks as they are computationally inefficient. However, practical algorithms may only need to leverage the insight that optimism is helpful and may not require such thorough exploration. Instead, we could construct a bonus and add it to the value function, and then choose the greedy policy for the optimistic value function \citep{curi2020efficient}. 

When computing the weights, calculating the uncertainty quantity in practical scenarios is the main difficulty. Inspired by \citet{ye2023corruptionb}, we can approximate the uncertainty by "bootstrapped uncertainty". The intuition is that if the transition probability can be embedded into a vector space $P_M^h(x^{h+1}|z^h) = \nu^h(M,x^{h+1})^{\top}\phi^h(z^h)$. According to Example \ref{eg:Information Ratio for Linear Model}, the uncertainty can be reduced to $\|\phi^h(z^h)\|_{(\Sigma_t^h)^{-1}}$, where the covariance matrix $\Sigma_t^h = \sum_{s=1}^{t-1}\phi^h(z_s^h)\phi^h(z_s^h)^{\top}$.
Hence, from \citet{ye2023corruptionb}, the bootstrapped variance (called bootstrapped uncertainty) is an estimation of the uncertainty. To compute the bootstrapped uncertainty, we first learn $K$ transition probabilities independently with different seeds. Then, we take the uncertainty estimation as $\sqrt{\mathrm{Var}_{i=1,\ldots,K}[P_{M_k}(z)]}$. We leave the development of practical algorithms as future work.

\subsection{Regret Bounds}
The following two theorems offer theoretical guarantees for the upper and lower bounds of regret under the CR-OMLE algorithm.
\begin{theorem}[Upper Bound]\label{th:online_upper}
Under Assumption \ref{as:Bounded Condition}, given a finite eluder dimension $\ED(\cM,\epsilon)$, if we choose $\beta=  5\sqrt{\log(|\cM|/\delta)\log^2 B} + 7\alpha C$
and $\lambda=\log|\cM|$, $\alpha=\sqrt{\log|\cM|\log^2 B}/C$, with probability at least $1-\delta$, the regret of Algorithm \ref{alg:CR-OMLE} is upper bounded by
\$
    \reg(T)
    =& \tcO\Big(H\log B\sqrt{T\log|\cM|\ED(\cM,\sqrt{\lambda/T})}\\
    &\quad + CH\cdot\ED(\cM,\sqrt{\lambda/T})\Big).
\$
\end{theorem}
This regret bound consists of two parts. The main part $H(T\log B\log|\cM|\ED(\cM,\sqrt{\lambda/T}))^{1/2}$ is unrelated to corruption and matches the bound for OMLE \citep{liu2023optimistic}. According to Theorem \ref{th:TV-Eluder Dimension for Linear MDPs}, in the linear MDP setting with $d$ dimensions, $\ED(\cM,\sqrt{\lambda/T})=O(d)$. Thus, the regret reduces to $\tcO(H\log B\sqrt{Td\log|\cM|} + CHd)$, where the corruption part $CHd$ also matches the lower bound from Theorem \ref{th:online_lower}. We can achieve sub-linear regret when the corruption level is sub-linear. We highlight the proof sketch in the sequel and delay the detailed proof to Appendix \ref{ss:Proof of Theorem online_upper}.

\begin{theorem}[Lower Bound]\label{th:online_lower}
For any dimension $d\ge 2$, stage $H\ge 3$ and a known corruption level $C>0$, if the number of episode $T$ satisfied $T\ge \Omega (dCH +H^2)$, there exists an instance such that any algorithm must incur $(H-2)(d-1)C/64$ expected regret.
\end{theorem}
\begin{remark}
It is noteworthy that the lower bound in Theorem \ref{th:online_lower} matches the corruption term $dCH$ in our upper bound, up to logarithmic factors. This result suggests that our algorithm is optimal for defending against potential adversarial attacks. Furthermore, for any algorithm, the regret across the first $T$ episodes is limited to no more than $\Omega(T)$. In this situation, the requirement that the number of episodes $T>\Omega (dCH)$ is necessary to achieve a lower bound of $\tilde \Omega(dCH)$. The proof is available in Appendix \ref{ss:Proof of Theorem online_lower1}.
\end{remark}

\paragraph{Unknown Corruption Level}
Since it is hard to know the corruption level ahead of time, we discuss the solution for the unknown corruption level. Note that only the choice of parameter $\alpha=\sqrt{\log|\cM|\log^2 B}/C$ requires the knowledge of $C$, so we following \citet{he2022nearly} to replace $C$ in $\alpha$ by a predefined corruption tolerance threshold $\bar C$. As long as the actually corruption level $C$ is no more then than the tolerance threshold $\bar{C}$, we can still obtain a non-trivial regret bound.
\begin{theorem}[Unknown Corruption Level]\label{th: Unknown Corruption Level}
Under the same conditions as Theorem \ref{th:online_upper}, let $\beta_t^h=\Theta(\sqrt{\log(|\cM|/\delta)\log^2 B})$ and $\alpha=\sqrt{\log|\cM|}/\bar C$. If $C \le \bar C$, we have with probability at least $1-\delta$,
\$
    \reg(T)
    =& \tcO\Big(H\log B\sqrt{T\log|\cM|\ED(\cM,\sqrt{\lambda/T})}\\
    &\quad + \bar C H\cdot\ED(\cM,\sqrt{\lambda/T})\Big).
\$
On the other hand, if $C>\bar C$, we have $\reg(T)=\Tilde{\cO}(T)$.
\end{theorem}
\begin{remark}
    Theorem \ref{th: Unknown Corruption Level} establishes a trade-off between adversarial defense and algorithmic performance.  Specifically, a higher predefined corruption tolerance threshold $\bar C$ can effectively fortify against a broader spectrum of attacks. However, this advantage comes at the cost of a more substantial regret guarantee. For a notable case, when we set $\bar{C} ={ \sqrt{T\log |\cM|/\ED(\cM,\sqrt{\lambda/T})}}$, our algorithm demonstrates a similar performance as the no corruption case, even in the absence of prior information, within the specified region of $C\leq \bar{C}.$ We defer the proof to Appendix \ref{ss:th Unknown Corruption Level}.
\end{remark}

\subsection{Proof Sketch.} 
We provide the proof sketch for Theorem \ref{th:online_upper}, which consists of three key steps.

\paragraph{Step I. Regret Decomposition.}
Following the optimism principle, we can decompose regret as the sum of Bellman errors and corruption:
\$
\reg(T) \le \sum_{t=1}^T\E^t_{\pi_t}\sum_{h=1}^H\cE^h(M_t,z_t^h) + \sum_{t=1}^T\E^t_{\pi_t} \sum_{h=1}^H c_t^h(x^h,a^h).
\$
The summation of corruption occurs because, when transforming the regret into Bellman error, we must consider the Bellman error under the distribution of real data, i.e., the corrupted one.

\paragraph{Step II. Confidence Set for Optimism.} 
To ensure that the true, uncorrupted model $M^*$ belongs to the confidence set $\cM_t$, we demonstrate in the following lemma that $M^*$ satisfies \eqref{eq:confidence_set}. Moreover, for any $M\in\cM_t$, the in-sample error is bounded.
\begin{lemma}[Confidence Set]\label{lm:confidence_set}
Under Assumption \ref{as:Bounded Condition}, if we choose
$
    \beta= 5\sqrt{\log(|\cM|/\delta)\log^2 B} + 7\alpha C
$ in Algorithm \ref{alg:CR-OMLE}, then
with probability at least $1-\delta$, for all $h\in[H]$ and all $t\in[T]$, $M_*\in\cM_t$, we have
\$
    \sum_{s=1}^{t-1} \tv(P_*^h(\cdot|x_s^h,a_s^h)\|P_{\bar M_t}^h(\cdot|x_s^h,a_s^h))^2/\sigma_s^h \le 2\beta^2.
\$
Moreover, for any $M\in\cM_t$, with probability at least $1-\delta$, we have
\$
    \sum_{s=1}^{t-1} \tv(P_M^h(\cdot|x_s^h,a_s^h)\|P_{\bar M_t}^h(\cdot|x_s^h,a_s^h))^2/\sigma_s^h \le 4\beta^2.
\$
\end{lemma}
The key ideas of the analysis are presented in Subsection \ref{ss:Uncertainty Weighting for Model-based RL}, and we postpone the detailed explanation to Appendix \ref{ss:Lemmas for Online Setting}.

\paragraph{Step III. Bounding the sum of Bellman Errors.} 
Finally, by combining the results derived from the first two steps and categorizing the samples into two classes based on whether $\sigma_t^h>1$, we obtain an upper bound on the regret:
{\small
\$
\tcO\Bigg(&\sqrt{TH\log|\cM|\sum_{h=1}^H\sup_{\cS_T}\sum_{t=1}^T\E^t_{\pi_t} \big(I^h_{\sigma}(\lambda,\cM_t,\cS_t)\big)^2}\\
&+ C\sum_{h=1}^H\sup_{\cS_T}\sum_{t=1}^T\E^t_{\pi_t} \big(I^h_{\sigma}(\lambda,\cM_t,\cS_t)\big)^2\Bigg),
\$}
where we define the weighted form of IR:
{\small
\#\label{eq:tv_information_ratio_weighted}
&I^h_{\sigma}(\lambda,\cM,\cS_t^h)=\min\bigg\{1,\sup_{M\in\cM_t}\notag\\
&\quad\frac{ \tv(P_{M}^h(\cdot|z_t^h)\|P_{\bar M_t}^h(\cdot|z_t^h))/(\sigma_t^h)^{1/2}}{\sqrt{\lambda+\sum_{s=1}^{t-1}\tv(P_{M}^h(\cdot|z_s^h)\|P_{\bar M_t}^h(\cdot|z_s^h))^2/\sigma_s^h}}\bigg\}.
\#}
Moreover, to establish an instance-independent bound, we follow \citet{ye2023corruptiona} to demonstrate the relationship between the sum of weighted Information Ratio (IR) concerning dataset $\cS_t={z_s^h}_{s=1}^t$ and the eluder dimension:
\$
\sup_{\cS_T^h}\sum_{t=1}^T I_{\sigma}^h(\lambda,\cM,\cS_t)^2 = \tO(\ED(\cM,\sqrt{\lambda/T})),
\$
which leads to the final bound.

\section{Offline Model-based RL}

In this section, we will extend the uncertainty weighting technique proposed in \ref{sec:online} to the setting of offline RL, and propose a corruption-robust model-based offline RL algorithm and its analysis.

\subsection{Algorithm: CR-PMLE}
\begin{algorithm}[th]
\small
\caption{Corruption-Robust Pessimistic MLE (CR-PMLE)}
\label{alg:Corruption-Robust Pessimistic MLE}
\begin{algorithmic}[1]
\STATE {\bf Input:} $\mathcal{D}=\{(x_t^h, a_t^h, r_t^h)\}_{t,h=1}^{T,H},\cM$.
\STATE For $h=1,\ldots,H$, choose weights $\{\sigma_t^h\}_{t=1}^T$ by proceeding Algorithm \ref{alg:wi_off} with inputs $\{(x_t^h,a_t^h)\}_{t=1}^T, \cM, \alpha$;
\STATE Let 
\$
\bar{M} = \argmax_{M\in\cM}\sum_{t=1}^T\sum_{h=1}^H (\sigma_t^h)^{-1}\log P_M^h(x_t^{h+1}|x_t^h,a_t^h);
\$
\STATE Find $\beta$ and construct confidence set $\hcM$
\$
\Big\{&M\in\cM: \forall~h\in[H], ~\sum_{t=1}^T \frac{\log P_M^h(x_t^{h+1}|x_t^h,a_t^h)}{\sigma_t^h}\\
&\quad\ge \sum_{t=1}^T \frac{\log P_{\bar{M}}^h(x_t^{h+1}|x_t^h,a_t^h)}{\sigma_t^h} - \beta^2\Big\};
\$
\STATE Set $(\hat{\pi},\hM) = \argmax_{\pi\in\Pi}\min_{M\in\hcM}V_{M,\pi}^{1}$;
\STATE {\bf Output:} $\{\hat{\pi}^h\}_{h=1}^H$.
\end{algorithmic}
\end{algorithm}

In the offline learning setting, we introduce our algorithm, Corruption-Robust Pessimistic MLE, shown in Algorithm \ref{alg:Corruption-Robust Pessimistic MLE}. In this algorithm, we also estimate the model by maximizing the log-likelihood with uncertainty weights and constructing the confidence set $\hcM$. To address the challenge that weights cannot be computed iteratively in rounds, as in the online setting, we adopt the weight iteration algorithm (Algorithm \ref{alg:wi_off}) from \citet{ye2023corruptionb} to obtain an approximated uncertainty. They prove that this iteration will converge since the weights are monotonically increasing and upper-bounded. Moreover, as long as the weights $\sigma_t^h$ are in the same order as the truncated uncertainty, the learning remains robust to corruption. For completeness, we present the convergence result in Lemma \ref{lm:converge_weight}.

\begin{algorithm}[th]
\small
\caption{Uncertainty Weight Iteration}
\label{alg:wi_off}
\begin{algorithmic}[1]
\STATE {\bf Input:} $\{(x_t^h, a_t^h)\}_{t=1}^T$,$\cM$,$\alpha>0$.
\STATE {\bf Initialization:} $k=0,~\sigma_t^0=1$, $t=1,\ldots,T$.
\REPEAT
\STATE $k\leftarrow k+1$;
\STATE For $t=1,\ldots,T$, let
\$
\sigma_t^k \leftarrow  \max\Big(&1,\sup_{M,M'\in\cM}\\
&\frac{l_t^h(M,M')/\alpha}{\sqrt{\lambda + \sum_{s=1}^Tl_s^h(M,M')^2/\sigma_s^{k-1}}}\Big);
\$
\UNTIL $\max_{t\in[T]} \sigma_t^{k}/\sigma_t^{k-1} \le 2$;
\STATE {\bf Output:} $\{\sigma_t^k\}_{t=1}^T$.
\end{algorithmic}
\end{algorithm}

To consider the case where the offline data lacks full coverage, we proceed with pessimism and tackle a minimax optimization
\$
(\hpi,\hM) = \argmax_{\pi\in\Pi}\min_{M\in\hcM}V_{M,\pi}^{\pi},
\$
which shares a similar spirit with the model-free literature \citep{xie2021bellman}.

\subsection{Analysis}
In this subsection, we first provide an instance-dependence suboptimality upper bound. Then, we can obtain an instance-independent bound under the uniform coverage condition. Finally, we present a lower bound under the uniform coverage condition, which aligns with the corruption term in the upper bound.
\paragraph{Instance-Dependent Bound.}
To facilitate analysis, we define the offline variant of IR \eqref{eq:tv_information_ratio} as follows.
\begin{definition}\label{df:information coefficient}
Given an offline dataset $\cD$, for any initial state $x^1=x\in\cX$, the information coefficient with respect to space $\hcM$ is:
{\small
\$
&\IC^{\sigma}(\lambda,\hcM,\cD)\\
&= \max_{h\in[H]} \E_{\pi_*}\bigg[\sup_{M,M'\in\hcM} \frac{T\cdot l^h(M,M',z^h)^2/\sigma^h(z^h)}{\lambda + \sum_{t=1}^Tl_t^h(M,M')^2/\sigma_t^h}\bigg],
\$}
where we define $l^h(M,M',z^h)=\tv(P_M^h(\cdot|z^h)\|P_{M'}^h(\cdot|z^h))$, $\E_{\pi_*}$ is taken with respect to $(x^h,a^h)$ induced by policy $\pi_*$ for the uncorrupted transition, and we define
{\small
\$
\sigma^h(z^h) = \max\bigg\{1, \sup_{M,M'\in\hcM}\frac{l^h(M,M',z^h)/\alpha}{\sqrt{\lambda + \sum_{t=1}^Tl_t^h(M,M')^2/\sigma_t^h}}\bigg\}.
\$}
\end{definition}
This coefficient depicts the information ratio of the trajectory for the optimal policy $\pi_*$ and the dataset. Now, we can demonstrate an instance-dependent bound, where uniform data coverage is not required.
\begin{theorem}[Instance-Dependent Bound]\label{th:Instance-Dependent Bound}
Suppose that Assumption \ref{as:Bounded Condition} holds. Under Algorithm \ref{alg:Corruption-Robust Pessimistic MLE}, if we choose $\lambda=\log |\cM|$, $\alpha=\sqrt{\log|\cM|\log^2 B}/C$ and
\$
\beta= 5\sqrt{\log(|\cM|/\delta)\log^2 B} + 7\alpha C.
\$
Then, with probability at least $1-2\delta$, the sub-optimality $\subopt(\hat\pi, x)$ is bounded by
{\small
\$
\tcO\bigg(&H\log B\sqrt{\frac{\IC^{\sigma}(\lambda,\hcM,\cD)\log|\cM|}{T}}+ \frac{\IC^{\sigma}(\lambda,\hcM,\cD)CH}{T}\bigg),
\$}
where the weighted information coefficient $\IC^{\sigma}(\lambda,\hcM,\cD)$ is defined in Definition \ref{df:information coefficient}.
\end{theorem}
This bound depends on the information coefficient with uncertainty weights along the trajectory induced by $\pi_*$. To remove the weights $\sigma_t^h$ from the suboptimality bound, we need a stronger coverage condition. 

We present the proof in Appendix \ref{ss:Proof of Theorem Instance-Dependent Bound}, and highlight key points here. The proof also contains three steps. First of all, with pessimism, the suboptimality can be decomposed as
\$
\E_{\pi_*}\sum_{h=1}^H\cE^h(\hM,z^h) \le \E_{\pi_*}\sum_{h=1}^H \tv(P_{\hM}^h(\cdot|z^h)\|P_{ M_*}^h(\cdot|z^h)).
\$
Given the proximity of the weights $\sigma_t^h$ to uncertainty measures, we can modify Lemma \ref{lm:confidence_set} to achieve $\beta=\Theta(\sqrt{\log(|\cM|/\delta)\log^2 B} + \alpha C)$ for $\delta>0$. This modification allows us to establish the high-probability inclusion of $M_*$ in $\hcM$. Ultimately, by integrating the initial two steps and leveraging the relationship between uncertainty weights and the information coefficient, we derive the bound.

\paragraph{Instance-Independent Bound.}
To obtain an instance-independent upper bound, we follow \citet{ye2023corruptionb} to introduce a TV-norm-based version of the coverage condition for the dataset.
\begin{assumption}[Uniform Data Coverage]\label{as:offline_upper}
Define the measure for any $M,M'\in\cM$, $\rho^h(M,M')=\sup_{z}\tv(P^h_{M}(\cdot|z)\|P^h_{M'}(\cdot|z))$.
There exists a constant $\Cov(\cM,\cD)$ such that for any $h\in[H]$, and two distinct $M,M'\in\cM$,
{\small
\#\label{eq:condition of lm:Connections between Weighted and Unweighted Coefficient}
&\frac{1}{T}\sum_{t=1}^T\tv(P_{M}^h(\cdot|z_t^h)\|P_{M'}^h(\cdot|z_t^h))^2\notag \ge \Cov(\cM,\cD)\big(\rho^h(M,M')\big)^2.
\#}
\end{assumption}
This is a generally adopted condition in offline literature \citep{duan2020minimax,
wang2020statistical,
xiong2022nearly,di2023pessimistic}. Intuitively, this condition depicts that the dataset explores each direction of the space. In linear case where $\cM^h=\{\langle \nu^h(M,x^{h+1}),\phi^h(\cdot) \rangle : \cX\rightarrow\rR\}$, according to Lemma \ref{lm:minimum eigenvalue condition and condition for W and UW connections}, this assumption is implied by the condition that the covariance matrix $T^{-1}\sum_{t=1}^T \phi(z_t^h)\phi(z_t^h)^\top$ is strictly positive definite, and $\Cov(\cM,\cD)$ has a $\Theta(d^{-1})$ dependence on $d$. Under this condition, the suboptimality upper bound yielded by Algorithm \ref{alg:Corruption-Robust Pessimistic MLE} is guaranteed.
\begin{theorem}[Instance-Independent Bound]\label{th:offline_upper}
Supposing that Assumption \ref{as:Bounded Condition} and \ref{as:offline_upper} hold, if we choose $\lambda=\log |\cM|$, $\alpha=\sqrt{\log|\cM|\log^2 B}/C$ and
\$
\beta= 5\sqrt{\log(|\cM|/\delta)\log^2 B} + 7\alpha C.
\$
Then, with probability at least $1-2\delta$, the sub-optimality $\subopt(\hat\pi, x)$ of Algorithm \ref{alg:Corruption-Robust Pessimistic MLE} is bounded by
\$
\tcO\bigg(H\log B\sqrt{\frac{\log |\cM|}{T\Cov(\cM,\cD)}} + \frac{CH}{T\Cov(\cM,\cD)}\bigg).
\$
\end{theorem}
We present the proof in Appendix \ref{ss:Proof of Theorem offline_upper}.
Additionally, when the corruption level is unknown, we can take the hyper-parameter $\alpha$ as a tuning parameter and attain a similar result as Theorem \ref{th: Unknown Corruption Level}.

\begin{remark}[Comparison between Online and Offline Results]
Both of the online and offline bounds are composed of the main term (uncorrupted error) and the corruption error, and the corruption error is sublinear as long as the corruption level $C$ is sublinear w.r.t sample size $T$. The main difference is that the online bound is affected by the eluder dimension and the offline bound is affected by the covering coefficient. This difference comes from the ability to continue interacting with the environment. In the online setting, the agent can explore the model space extensively, thus, the final regret bound is controlled by the complexity of the model space.
In contrast, exploration in the offline setting is constrained by a given dataset, thus the final bound is determined by the coverage quality of the dataset.
\end{remark}

Regardless of the different conditions between online and offline cases, this bound roughly
aligns with the online upper bound. Specifically, achieving an $\epsilon$-suboptimality under both online and offline algorithms requires a $\tcO(\epsilon^{-2}H\log B\sqrt{\log|\cM|\text{dim}} + \epsilon^{-1}CH\text{dim})$ sample complexity. The $\text{dim}$ denotes the eluder dimension for the online setting and inverse of coverage $(\Cov(\cM,\cD))^{-1}$ for the offline setting. Particularly, for the linear case, the $\text{dim}=d$ for both settings. Additionally, the corruption component of the bound closely corresponds to the demonstrated lower bound below.

\begin{theorem}[Lower Bound]\label{th:offline_lower}
For a given corruption level $C$, data coverage coefficient $\Cov(\cM,\cD)$, dimension $d>3$ and episode length $H>2$, it the data size $T$ satisfied $T>\Omega \big(\Cov(\cM,\cD)/C\big)$, there exist a hard to learn instance such that the suboptimality for any algorithm is lower bounded by
\begin{align*}
    \mathbb{E}[\subopt(\hat\pi, x)]\ge  \Omega\Big(\frac{C}{\Cov(\cM,\cD) T}\Big). 
\end{align*}
\end{theorem}
\begin{remark}
    The offline lower bound in Theorem \ref{th:offline_lower} matches the corruption term $CH/\big(T\Cov(\cM,\cD)\big)$ in our upper bound, up to a factor of $H$. This result demonstrates that our algorithm achieves near-optimal suboptimality guarantee for defending against adversarial attacks. 
\end{remark}
In comparison to the online lower bound presented in Theorem \ref{th:online_upper}, a discrepancy of $H$ emerges, and we believe this variance comes from the data coverage coefficient $\Cov(\cM,\cD)$. In general, meeting the uniform data coverage assumption becomes more challenging with an increased episode length $H$, resulting in a smaller data coverage coefficient $\Cov(\cM,\cD)$. Furthermore, for the hard-to-learn instances that constructed in the lower bound, the coefficient exhibits an inverse dependence on the episode length $H$, specifically, $\Cov(\cM,\cD)=O(1/H)$. Under this situation, our lower bound does not incur additional dependencies on $H$ and we leave it as a potential future work.
More details can be found in Appendix \ref{sec-offline-lower}.

\section{Conclusion and Future Work}
We delve into adversarial corruption for model-based RL, encompassing both online and offline settings. Our approach involves quantifying corruption by summing the differences between true and corrupted transitions. In the online setting, our algorithm, CR-OMLE, combines optimism and weighted log-likelihood principles, with weights employing a TV-norm-based uncertainty. Our analysis yields a $O(\sqrt{T}+C)$ regret upper bound for CR-OMLE and establishes a new lower bound for MDPS with transition corruption, aligning in terms of corruption-dependent terms. In the offline realm, we present CR-PMLE, a fusion of pessimism and MLE with uncertainty weights, offering theoretical assurances through instance-dependent and instance-independent upper bounds. Notably, the instance-independent bound necessitates a uniform coverage condition, where the corruption-dependent term nearly matches the lower bound. We anticipate that our findings will contribute theoretical insights to address practical challenges in model-based RL, particularly for adversarial corruption.

Several future works are worth exploring: (1) Extending the uncertainty weighting technique to representation learning problems in RL \citep{jiang2017contextual,liu2022partially} would be an intriguing direction; and (2) In situations where the corruption level $C$ is unknown, our method relies on an optimistic estimation of corruption denoted by $\bar{C}$. There is a need for further investigation into the applicability of the model selection technique introduced by \citet{wei2022model} to the realm of model-based RL.

\section*{Acknowledgements}
The authors would like to thank the anonymous reviewers
for many insightful comments and suggestions. JH and QG are supported by the National Science Foundation CAREER Award 1906169 and research fund from UCLA-Amazon Science Hub. The views and conclusions contained in this paper are those of the authors and should not be interpreted as representing any funding agencies.


\section*{Impact Statement}
This paper introduces research aimed at enhancing the resilience of model-based reinforcement learning (RL) when confronted with adversarial corruption and malicious attacks. The enhanced robustness achieved by this work not only enhances the reliability of RL systems but also contributes to the development of more secure and trustworthy AI technologies. It will also promote greater public trust in AI applications, thereby promoting their responsible and ethical deployment across various domains.

\bibliography{ref.bib}
\bibliographystyle{icml2024}

\newpage
\appendix
\onecolumn
\section{Proof of Online Setting}

\subsection{Proof of Theorem \ref{th:online_upper}}\label{ss:Proof of Theorem online_upper}
We delay the proof of the supporting lemmas in Appendix \ref{ss:Lemmas for Online Setting}

\begin{lemma}[Bellman Decomposition]\label{lm:bellman_decomposition}
For any $M\in\cM$ and any $t\in[t]$, we have
\$
    V_M^1(x^1)-V_{\pi_M}^1(x^1)\le\E^t_{\pi_M}\sum_{h=1}^H\big[\cE^h(M,x^h,a^h) + c_t^h(x^h,a^h)\big],
\$
and 
\$
    V_M^1(x^1)-V_{\pi_M}^1(x^1)\ge\E^t_{\pi_M}\sum_{h=1}^H\big[\cE^h(M,x^h,a^h) - c_t^h(x^h,a^h)\big].
\$
\end{lemma}

We demonstrate the relationship between the weighted IR defined in \eqref{eq:tv_information_ratio_weighted} and eluder dimension in Definition \ref{df:tv-Eluder Dimension}. For simplicity, we consider a single step in the following lemma.

\begin{lemma}[Relation between Information Ratio and Eluder Dimension]\label{lm:Relation between Information Ratio and Eluder Dimension}
Consider a model class $\cM$ and sample set $\cS_T=\{z_t\}_{t=1}^T$. Let $\alpha=\sqrt{\log|\cM|}/C$ and $\lambda=\log|\cM|$. The square sum of weighted IR with weight established in Algorithm \ref{alg:CR-OMLE}
\$
I_{\sigma}(\lambda,\cM,\cS_t) = \sup_{M\in\cM_t}\min\bigg\{1,\frac{ \tv(P_{M}(\cdot|z_t)\|P_{\bar M_t}(\cdot|z_t))/\sigma_t^{1/2}}{\sqrt{\lambda+\sum_{s=1}^{t-1}\tv(P_{M}(\cdot|z_s)\|P_{\bar M_t}(\cdot|z_s))^2/\sigma_s}}\bigg\},
\$
\$
\sup_{\cS_T}\dim_{E,\sigma}(\lambda,\cM,\cS_T) =  \sup_{\cS_T}\sum_{t=1}^T I_{\sigma}(\lambda,\cM,\cS_t)^2 = \tO(\ED(\cM,\sqrt{\lambda/T}))
\$
\end{lemma}

Ultimately, we can prove Theorem \ref{th:online_upper}.
\begin{proof}[Proof of Theorem \ref{th:online_upper}]
Define the event
\[
A_1 = \Big\{M_*\in\cM_t,~\text{and}~\sum_{s=1}^{t-1} \tv(P_*^h(\cdot|z_s^h)\|P_{\bar M_t}^h(\cdot|z_s^h))^2/\sigma_s^h \le 2\beta^2,~\forall~t\in[T]\Big\}.
\]
According to Lemma \ref{lm:bellman_decomposition}, we know that $A_1$ holds with probability at least $1-\delta$.
Assuming that $A_1$ holds and using Lemma \ref{lm:bellman_decomposition}, we have
\$
    \reg(T) &= \sum_{t=1}^T V^1_*(x_t^1)-V_{M_t}^1(x_t^1)+V_{M_t}^1(x_t^1)-V_{\pi_t}^1(x_t^1)\notag\\
    &\le \sum_{t=1}^T V_{M_t}^1(x_t^1)-V_{\pi_t}^1(x_t^1)\notag\\
    &\le \sum_{t=1}^T\E^t_{\pi_t}\sum_{h=1}^H\cE^h(M_t,z_t^h) + \sum_{t=1}^T\E^t_{\pi_t} \sum_{h=1}^H c_t^h(x^h,a^h)\notag\\
    &\le \sum_{t=1}^T\E^t_{\pi_t}\sum_{h=1}^H \big[\E^{M_t}[V_{M_t}^{h+1}(x^{h+1})|z_t^h]-\E^* [V_{M_t}^{h+1}(x^{h+1})|z_t^h]\big] + CH\\
    &\le \sum_{t=1}^T\E^t_{\pi_t}\sum_{h=1}^H\tv\big(P_{M_t}^h(\cdot|z_t^h)\|P_{\bar M_t}^h(\cdot|z_t^h)\big) + \sum_{t=1}^T\E^t_{\pi_t}\sum_{h=1}^H\tv\big(P_*^h(\cdot|z_t^h)\|P_{\bar M_t}^h(\cdot|z_t^h)\big) + CH,
\$
where the first inequality uses $M_*\in \cM_t$ and $M_t=\argmax_{M\in\cM_t}V_M^1(x_t^1)$, and the last inequality uses $V_{M_t}^{h+1}(\cdot)\in[0,1]$.
For any $M\in\cM_t$, we have with probability at least $1-\delta$,
\#\label{eq:ap_aai}
    &\sum_{t=1}^T\E^t_{\pi_t}\sum_{h=1}^H\tv\big(P_M^h(\cdot|z_t^h)\|P_{\bar M_t}^h(\cdot|z_t^h)\big)\notag\\
    &\qquad\le \sum_{h=1}^H\sum_{t=1}^T\E^t_{\pi_t}\min\Bigg\{1,\frac{\tv\big(P_M^h(\cdot|z_t^h)\|P_{\bar M_t}^h(\cdot|z_t^h)\big)}{\sqrt{\lambda+\sum_{s=1}^{t-1}\tv\big(P_M^h(\cdot|z_s^h)\|P_{\bar M_t}^h(\cdot|z_s^h)\big)^2/\sigma_s^h}}\sqrt{\lambda+\sum_{s=1}^{t-1}\frac{\tv\big(P_M^h(\cdot|z_s^h)\|P_{\bar M_t}^h(\cdot|z_s^h)\big)^2}{\sigma_s^h}}\Bigg\}\notag\\
    &\qquad\le 3\sum_{h=1}^H\beta\sum_{t=1}^T\E^t_{\pi_t}\min\Bigg\{1,\frac{\tv\big(P_M^h(\cdot|z_t^h)\|P_{\bar M_t}^h(\cdot|z_t^h)\big)}{\sqrt{\lambda+\sum_{s=1}^{t-1}\tv\big(P_M^h(\cdot|z_s^h)\|P_{\bar M_t}^h(\cdot|z_s^h)\big)^2/\sigma_s^h}}\Bigg\}\notag\\
    &\qquad= 3\underbrace{\sum_{h=1}^H\beta\sum_{t=1}^T\E^t_{\pi_t}\min\Bigg\{1,\frac{\tv\big(P_M^h(\cdot|z_t^h)\|P_{\bar M_t}^h(\cdot|z_t^h)\big)}{\sqrt{\lambda+\sum_{s=1}^{t-1}\tv\big(P_M^h(\cdot|z_s^h)\|P_{\bar M_t}^h(\cdot|z_s^h)\big)^2/\sigma_s^h}}\mathbbm{1}\{\sigma_t^h=1\}\Bigg\}}_{P_1}\notag\\
    &\qquad\qquad + 3\underbrace{\sum_{h=1}^H\beta\sum_{t=1}^T\E^t_{\pi_t}\min\Bigg\{1,\frac{\tv\big(P_M^h(\cdot|z_t^h)\|P_{\bar M_t}^h(\cdot|z_t^h)\big)}{\sqrt{\lambda+\sum_{s=1}^{t-1}\tv\big(P_M^h(\cdot|z_s^h)\|P_{\bar M_t}^h(\cdot|z_s^h)\big)^2/\sigma_s^h}}\mathbbm{1}\{\sigma_t^h>1\}}_{P_2}\Bigg\},
\#
where the second inequality holds by invoking Lemma \ref{lm:confidence_set}:
\begin{align*}
\sum_{s=1}^{t-1}\tv\big(P_M^h(\cdot|z_s^h)\|P_{\bar M_t}^h(\cdot|z_s^h)\big)^2/\sigma_s^h \le 4\beta^2.
\end{align*}
Then, we bound terms $P_1$ and $P_2$ separately. For term $P_1$,
\#\label{eq:ap_aaj}
    P_1 &= \sum_{h=1}^H\sum_{t=1}^T\beta \E^t_{\pi_t} \min\Bigg\{1,\frac{\tv\big(P_M^h(\cdot|z_t^h)\|P_{\bar M_t}^h(\cdot|z_t^h)\big)/(\sigma_t^h)^{1/2}}{\sqrt{\lambda+\sum_{s=1}^{t-1}\tv\big(P_M^h(\cdot|z_s^h)\|P_{\bar M_t}^h(\cdot|z_s^h)\big)^2/\sigma_s^h}}\mathbbm{1}\{\sigma_t^h=1\}\Bigg\}\notag\\
    &\le \sqrt{\sum_{h=1}^H\sum_{t=1}^T\beta^2}\cdot\sqrt{\sum_{t=1}^T\E^t_{\pi_t} \sum_{h=1}^H \big(I^h_{\sigma}(\lambda,\cM_t,\cS_t)\big)^2}\notag\\
    &\le \sqrt{TH}\Big(5\sqrt{\log(|\cM|/\delta)\log^2 B} + 7\alpha C\Big)\sqrt{\sum_{t=1}^T\E^t_{\pi_t} \sum_{h=1}^H \big(I^h_{\sigma}(\lambda,\cM_t,\cS_t)\big)^2}.
\#
For term $P_2$,
\#\label{eq:ap_aak}
    P_2 &= \sum_{h=1}^H\beta\sum_{t=1}^T\E^t_{\pi_t}\min\Bigg\{1,\frac{\tv\big(P_M^h(\cdot|z_t^h)\|P_{\bar M_t}^h(\cdot|z_t^h)\big)/(\sigma_t^h)^{1/2}}{\sqrt{\lambda+\sum_{s=1}^{t-1}\tv\big(P_M^h(\cdot|z_s^h)\|P_{\bar M_t}^h(\cdot|z_s^h)\big)^2/\sigma_s^h}}(\sigma_t^h)^{1/2}\mathbbm{1}\{\sigma_t^h>1\}\Bigg\}\notag\\
    &\le \sum_{h=1}^H\beta\sum_{t=1}^T\E^t_{\pi_t}\min\Bigg\{1,\sup_{M\in\cM_t^h}\bigg(\frac{\tv\big(P_M^h(\cdot|z_t^h)\|P_{\bar M_t}^h(\cdot|z_t^h)\big)/(\sigma_t^h)^{1/2}}{\sqrt{\lambda+\sum_{s=1}^{t-1}\tv\big(P_M^h(\cdot|z_s^h)\|P_{\bar M_t}^h(\cdot|z_s^h)\big)^2/\sigma_s^h}}\bigg)^2\frac{1}{\alpha}\Bigg\}\notag\\
    &= \frac{1}{\alpha}\sum_{h=1}^H\beta\sum_{t=1}^T\E^t_{\pi_t}\big(I^h_{\sigma}(\lambda,\cM_t,\cS_t)\big)^2\notag\\
    &= \frac{1}{\alpha}\Big(5\sqrt{\log(|\cM|/\delta)\log^2 B} + 7\alpha C\Big)\sum_{t=1}^T\E^t_{\pi_t}\sum_{h=1}^H \big(I^h_{\sigma}(\lambda,\cM_t,\cS_t)\big)^2,
\#
where the inequality uses the definition of $\sigma_t^h$ when $\sigma_t^h>1$:
\begin{align*}
(\sigma_t^h)^{1/2} = \frac{1}{\alpha}\sup_{M\in\cM_t}\frac{\tv\big(P_{M}^h(\cdot|x_t^h,a_t^h)\|P_{\bar M_t}^h(\cdot|x_t^h,a_t^h)\big)/(\sigma_t^h)^{1/2}}{\sqrt{\lambda+\sum_{s=1}^{t-1}\tv\big(P_{M}^h(\cdot|x_s^h,a_s^h)\|P_{\bar M_t}^h(\cdot|x_s^h,a_s^h)\big)^2/\sigma_{s}^h}}
\end{align*}
By taking \eqref{eq:ap_aaj} and \eqref{eq:ap_aak} back into \eqref{eq:ap_aai} and taking $M$ as $M_*$ and $M_t$, respectively, we obtain with probability at least $1-\delta$,
\$
\reg(T) &\le 6\sqrt{TH}\Big(5\sqrt{\log(|\cM|/\delta)\log^2 B} + 7\alpha C\Big)\sqrt{\sum_{t=1}^T\E^t_{\pi_t} \sum_{h=1}^H \big(I^h_{\sigma}(\lambda,\cM_t,\cS_t)\big)^2}\\
&\qquad +
\frac{6}{\alpha}\Big(5\sqrt{\log(|\cM|/\delta)\log^2 B} + 7\alpha C\Big)\sum_{t=1}^T\E^t_{\pi_t}\sum_{h=1}^H \big(I^h_{\sigma}(\lambda,\cM_t,\cS_t)\big)^2 + CH\\
&= \tcO\Bigg(\sqrt{TH\log|\cM|\sum_{h=1}^H\sup_{\cS_T}\sum_{t=1}^T\E^t_{\pi_t} \big(I^h_{\sigma}(\lambda,\cM_t,\cS_t)\big)^2} + C\sum_{h=1}^H\sup_{\cS_T}\sum_{t=1}^T\E^t_{\pi_t} \big(I^h_{\sigma}(\lambda,\cM_t,\cS_t)\big)^2\Bigg)\\
&= \tcO\Bigg(H\sqrt{T\log|\cM|\ED(\cM,\sqrt{\lambda/T})} + CH\cdot\ED(\cM,\sqrt{\lambda/T})\Bigg),
\$
where we take $\alpha=\sqrt{\log|\cM|\log^2 B}/C$.
\end{proof}

\subsection{Proof of Theorem \ref{th:online_lower}}\label{ss:Proof of Theorem online_lower1}
To prove the lower bound, we first create a series of hard-to-learn tabular MDP $(\cX,\cA,H,\Pb,r)$ with $4$ different state $x_0,x_1,x_2,x_3$, where $x_2,x_3$ are absorbing states, and $d$ different actions $\cA=\{a_1,...,a_d\}$. For initial state $x_0$ and each stage $h\in[H]$, the agent will stay at the state $x_0$ with probability $1-1/H$ and transition to the state $x_1$ with probability $1/H$. For state $x_1$ and each stage $h\in[H]$, the agent will only transit to state $x_2$ or $x_3$.
In addition, an optimal action $a^*_h$ is selected from the action set $\cA=\{a_1,...,a_d\}$ and the transition probability at stage $h$ can be denoted by
\$
&P_{a_h^*}^{h}(x_2|x_1,a)=
\begin{cases}
    \frac{3}{4}, & a = a_h^*\\
    \frac{1}{4}, & a \ne a_h^*,
\end{cases}\\
\$
The agent is rewarded with $1$ at the absorbing state $x_2$ during the final stage $H$, while receiving zero reward otherwise. Consequently, state $x_2$ can be regarded as the goal state and the optimal strategy at stage $h$ involves taking action $a^*_h$ to maximize the probability of reaching the goal state $x_2$. Consequently, taking a not-optimal action $a_h \ne a_h^*$ at state $x_1$ for stage $2\leq h\leq H-1$ will incur a $1/2$ regret in a episode.

For these hard-to-learn MDPs, we can divide the episodes to several groups $\cT_1,..,\cT_{H}$, where episodes can be categorized into distinct groups, denoted as $\cT_1, \ldots, \cT_{H}$, where $\cT_i (i\in[H-1])$ comprises episodes transitioning to state $x_1$ after taking action at stage $i$, and $\cT_H$ includes episodes where the agent never reaches state $x_1$. Intuitively, the learning problem can be seen as the task of learning $H$ distinct linear bandit problems, with the goal of determining the optimal action $a_h^*$ for each stage $h\in[H]$. Now, considering a fixed stage $2\leq h\leq H$, let $t_1,t_{2},...,t_{dC/2}\in \cT_{h-1}$ be the first $dC/2$ episodes that transition to state $x_1$ after taking action at stage $h-1$. To ensure an adequate number of episodes in each group $\cT_{h-1}$, we continue the learning problem indefinitely, while our analysis focuses only on the regret across the first $T$ episodes. In this scenario, it is guaranteed, with probability $1$, that each group $\cT_{h-1}$ consists of at least $dC/2$ episodes. The following lemma provides a lower bound for learning the optimal action $a_h^*$ in the presence of adversarial corruption.
\begin{lemma}\label{lemma:lower-stage}
    For each fixed stage $2\leq h\leq H-1$ and any algorithm $\textbf{Alg}$ with the knowledge of corruption level $C$, if the optimal action $a^*_h$ is uniformed random selected from the action set $\cA=\{a_1,...,a_d\}$, then the expected regret across episode $t_1,...,t_{dC/2}$ is at least $(d-1)C/16$.
\end{lemma}
\begin{proof}[Proof of Lemma \ref{lemma:lower-stage}]
    To proof the lower bound, we construct an auxiliary transition probability function $P^h_{0}$ such that
    \begin{align*}
        P^h_0(x_2|x_1,a)=\frac{1}4, & \forall a\in \cA.
    \end{align*}
    Now, the following corruption strategy is employed for the the transition probability $P_{a_h^*}^h$: if the optimal action $a^*_h$ is selected at stage $h$ and the total corruption level for stage $h$ up to the previous step is no more than $C-2$, then the adversary corrupts the transition probability $P_{a_h^*}^h$ to $P^h_0$ and the corruption level $c_t^h$ in this episode is $c_t^h=|P_{a_h^*}^h(x_2|x_1,a_h^*)/P^h_0(x_2|x_1,a_h^*)-1|=2$.
    In this case, regardless of the actual optimal action $a_h^*$, there is no distinction between $P{a_h^*}^h$ and $P^h_0$, unless the agent selects the optimal action $a_h^*$ at least $C/2$ times, and the adversary lacks sufficient corruption level to corrupt the transition probability.
Now, let's consider the execution of the algorithm $\textbf{Alg}$ on the uncorrupted transition probability $P^h_0$. By the pigeonhole principle, there exist at least $d/2$ different arms, whose expected selected time is less than $C/2$ times. Without loss of generality, we assume these actions are $a_1,...,a_{d/2}$. Then for each action $a_i (i\leq d/2)$, according to Markov inequality, with probability at least $1/2$, the number of episodes selecting $a_i$ at stage $h$ is less than $C/2$. Under this scenario, the performance of $\textbf{Alg}$ on transition probability $P_0^h$ is equivalent to its performance on transition probability $P_{a_i}^h$ and the regret in these episode is at least $(dC/2-C/2)\times 1/2=(d-1)C/4$. Therefore, if the optimal action $a_h^*$ is uniform randomly selected from the action set $\cA=\{a_1,...,a_d\}$, then the expected regret across episodes $t_1,..,t_{dC/2}$ is lower bound by
\begin{align*}
    \mathbb{E}\Big[\sum_{i=1}^{dC/2} \mathds{1}  (a_{t_i}^h \ne a_h^*) \cdot 1/2\Big]\ge \frac 12\times \frac 12 \times \frac{(d-1)C}4 =\frac{(d-1)C}{16}.
\end{align*}
Thus, we complete the proof of Lemma \ref{lemma:lower-stage}.    
\end{proof}

Lemma \ref{lemma:lower-stage} provides a lower bound on the expected regret for each stage $2\leq h\leq H-1$ over the first $dC/2$ visits to state $x_1$. Regrettably, Lemma \ref{lemma:lower-stage} does not directly yield a lower bound for the regret over the initial $T$ episodes, as the agent may not visit the state $x_1$ $dC/2$ times for some stage $h\in[H]$. The following lemma posits that as the number of episodes $T$ increases, the occurrence of this event becomes highly improbable.
\begin{lemma}\label{lemma:visit}
    For the proposed hard-to-learn MDPs, if the number of episodes $T$ satisfies $T > edCH+2e^2H^2\log(1/\delta)$, then  for each stage $2\leq h\leq H-1$, with a probability of at least $1-\delta$, the agent visits state $x_1$ at least $dC/2$ times.
\end{lemma}
\begin{proof}[Proof of Lemma \ref{lemma:visit}]
    For any episode $T$ and stage $h\in[H]$, the agent, starting from the current state $x_0$, will transition to state $x_1$ with a probability of $1/H$, regardless of the selected action. Consequently, the probability that the agent visits state $x_1$ at stage $h$ can be expressed as:
    \begin{align*}
        P(x_h=x_1)=\bigg(1-\frac{1}{H}\bigg)^{h-1} \cdot \frac{1}{H}\ge \frac{1}{eH}, \forall 2\leq h\leq H-1.
    \end{align*}
    For a fixed stage $2\leq h\leq H-1$, we define the random variable $y_i= \mathds{1} (x_{i}^h=x_1)$ as the indicator function of visiting the state $x_1$ at stage $h$ of episode $i$. Subsequently, leveraging the Azuma–Hoeffding inequality (Lemma \ref{lemma:azuma}), with a probability of at least $1-\delta$, we obtain:
    \begin{align*}
        \sum_{i=1}^T y_i\ge \frac{T}{eH}- \sqrt{2T\log(1/\delta)}\ge \frac{T}{2eH} - eH\log (1/\delta),
    \end{align*}
    where the last inequality holds due to $x^2+y^2\ge 2xy$. Therefore, for $T>edCH+ 2e^2H^2\log(1/\delta)$, with probability at least $1-\delta$, we have 
    \begin{align*}
        \sum_{i=1}^T y_i\ge dC/2,
    \end{align*}
    which completes the proof of Lemma \ref{lemma:visit}.
\end{proof}
With the help of these lemmas, we are able to prove Theorem \ref{th:online_lower}.
\begin{proof}[Proof of Theorem \ref{th:online_lower}]
    For each stage $2\leq h\leq H-1$, we use $\cE_h$ to denote the events that the agent visit the state $x_1$ at least $dC/2$ times for stage $h$ across the first $T$ episodes. Under this situation, we use $t_{h,1},...,t_{h,dC/2}$ denotes the first $dC/2$ episodes that visit the state $x_1$ at stage $h$.
    
    Then for any algorithm, the expected regret can be lower bounded by
\begin{align}
    \mathbb{E}[\text{Regret(T)}]&=\mathbb{E}\Big[\sum_{t=1}^T\sum_{h=2}^{H-1} \mathds{1} (x_{t}^h=x_1)\cdot \mathds{1}  (a_t^h \ne a_h^*) \cdot 1/2\Big]\notag\\
    &\ge \sum_{h=2}^{H-1} \mathds{1}(\cE_h) \mathbb{E} \Big[\sum_{t=1}^T \mathds{1} (x_{t}^h=x_1)\cdot \mathds{1}  (a_t^h \ne a_h^*) \cdot 1/2|\cE_h \Big]\notag\\
    &\ge \sum_{h=2}^{H-1} \mathds{1}(\cE_h)  \mathbb{E} \Big[\sum_{i=1}^{dC/2} \mathds{1}  (a_{t_{h,i}}^h \ne a_h^*) \cdot 1/2|\cE_h \Big]\notag\\
    &=\sum_{h=2}^{H-1} \mathbb{E} \Big[\sum_{i=1}^{dC/2} \mathds{1}  (a_{t_{h,i}}^h \ne a_h^*) \cdot 1/2 \Big] - \sum_{h=2}^{H-1} \mathds{1}( \neg \cE_h) \mathbb{E} \Big[\sum_{i=1}^{dC/2} \mathds{1}  (a_{t_{h,i}}^h \ne a_h^*) \cdot 1/2|\neg\cE_h \Big]\notag\\
    &\ge \frac{(H-2)(d-1)C}{32}- \sum_{h=2}^{H-1} \mathds{1}( \neg \cE_h) \mathbb{E} \Big[\sum_{i=1}^{dC/2} \mathds{1}  (a_{t_{h,i}}^h \ne a_h^*) \cdot 1/2|\neg\cE_h \Big]\notag\\ 
    &\ge \frac{(H-2)(d-1)C}{32}- \frac{(H-2)\delta dC}{4}\label{eq:lower-01}
\end{align}
    where the first equation holds due to the construction of these hard-to-learn MDPs, the first inequality holds due to the fact that $\mathbb{E}[x]\leq \mathbb{E}[x|\cE] \cdot \Pr(\cE)$, the second inequality holds due to the definition of event $\cE_h$, the third inequality holds due to Lemma \ref{lemma:lower-stage} and the second inequality holds due to Lemma \ref{lemma:visit}. Finally, setting the probability $\delta = 1/32$, we have
    \begin{align*}
   \mathbb{E}[\text{Regret(T)}]     &\ge \frac{(H-2)(d-1)C}{64}.
    \end{align*}
    Thus, we complete the proof of Theorem \ref{th:online_lower}. 
\end{proof}

\subsection{Unknown Corruption Level}\label{ss:th Unknown Corruption Level}
\begin{proof}[Proof of Theorem \ref{th: Unknown Corruption Level}]
First of all, we consider the case when $C\le \bar C$. Since only the hyper-parameter $\alpha$ is modified by replacing $C$ with $\bar C$ and $C\le \bar C$, we can follow the analysis of Theorem \ref{th:online_upper} to derive the suboptimality bound with probability at least $1-\delta$,
\$
\reg(T)
    =& \tcO\Big(H\log B\sqrt{T\log|\cM|\ED(\cM,\sqrt{\lambda/T})}\\
    &\quad + \bar C H\cdot\ED(\cM,\sqrt{\lambda/T})\Big).
\$
Additionally, we can also demonstrate the relationship between the sum of weighted information ratio and the eluder dimension as Lemma~\ref{lm:Relation between Information Ratio and Eluder Dimension} by discussing the value of $\bar C$ in the first step.

For the case when $C>\bar C$, we simply take the trivial bound $T$.
\end{proof}

\section{Proof of Offline Setting}

\subsection{Proof of Theorem \ref{th:Instance-Dependent Bound}}\label{ss:Proof of Theorem Instance-Dependent Bound}

First of all, we demonstrate that for the uncertainty weight iteration in Algorithm \ref{alg:wi_off} converges and the output solution approximates the real uncertainty quantity. Since the convergence has been illustrated in \citet{ye2023corruptionb}, we restate the result in the following lemma.

\begin{lemma}[Lemma 3.1 of \citet{ye2023corruptiona}]\label{lm:converge_weight}
There exists a $T$ such that the output of Algorithm \ref{alg:wi_off} $\{\sigma_t:=\sigma_t^{K+1}\}_{t=1}^T$ satisfies:
\#\label{eq:approximate_weight}
\sigma_t \ge \max\big\{1,\psi(z_t)/2\big\}, \quad \sigma_t \le \max\big\{1,\psi(z_t)\big\},
\#
where
\[
\psi(z_t)=\sup_{M,M'\in\cM}\frac{\tv(P_M^h(\cdot|z_t)\|P_{M'}^h(\cdot|z_t))/\alpha}{\sqrt{\lambda + \sum_{s=1}^T\tv(P_M^h(\cdot|z_s)\|P_{M'}^h(\cdot|z_s))^2/\sigma_s}}.
\]
\end{lemma}
We provide the proof in Appendix \ref{ss:Lemmas for Offline Setting}.

\begin{lemma}\label{lm:off_Confidence Set}
Under Assumption \ref{as:Bounded Condition} and Algorithm \ref{alg:Corruption-Robust Pessimistic MLE}, if we choose
\$
\beta= 5\sqrt{\log(|\cM|/\delta)\log^2 B} + 7\alpha C,
\$
we have with probability at least $1-\delta$, for all $h\in[H]$ and all $t\in[T]$, $M_*\in\hcM$ and
\$
    \sum_{t=1}^T \tv(P_*^h(\cdot|z_t^h)\|P_{\bar M}^h(\cdot|z_t^h))^2/\sigma_t^h \le 2\beta^2.
\$
Moreover, for any $M\in\cM_t$, we have with probability at least $1-\delta$,
\$
    \sum_{t=1}^T \tv(P_M^h(\cdot|z_t^h)\|P_{\bar M}^h(\cdot|z_t^h))^2/\sigma_t^h \le 4\beta^2.
\$
\end{lemma}
We show the proof in Appendix \ref{ss:Lemmas for Offline Setting}.

\begin{proof}[Proof of Theorem \ref{th:Instance-Dependent Bound}]
Define the event
\[
A_2 = \Big\{M_*\in\cM_t,~\text{and}~\sum_{t=1}^T \tv(P_*^h(\cdot|z_s^h)\|P_{\bar M_t}^h(\cdot|z_s^h))^2/\sigma_s^h \le 2\beta^2\Big\}.
\]
According to Lemma \ref{lm:bellman_decomposition}, we know that $A_1$ holds with probability at least $1-\delta$.
First of all, we obtain
\begin{align*}
\subopt(\hat{\pi},x^1) = & V_*^1(x^1) - V_{\hM}^1(x^1) + V_{\hM}^1(x^1) - V_{\hpi}^1(x^1)\\
\le & V_*^1(x^1) - V_{\hM}^1(x^1)\\
\le & \E_{\pi_*}\sum_{h=1}^H \big(V_{\hM,\pi_*}^h(x^h) - V_{\hM,\hpi}^h(x^h)\big) - \E_{\pi_*}\sum_{h=1}^H \cE^h(\hM,z^h)\\
\le & - \E_{\pi_*}\sum_{h=1}^H \big(\E^{\hM}[V_{\hM}^{h+1}(x^{h+1})|z^h] - \E^{M_*}[V_{\hM}^{h+1}(x^{h+1})|z^h]\big)\\
\le & \E_{\pi_*}\sum_{h=1}^H \tv(P_{\hM}^h(\cdot|z^h)\|P_{\bar M}^h(\cdot|z^h)) + \E_{\pi_*}\sum_{h=1}^H \tv(P_{M_*}^h(\cdot|z^h)\|P_{\bar M}^h(\cdot|z^h))
\end{align*}
and the first inequality uses the condition that $M_*\in\hcM$, which implies that $V_{\hcM,\hpi}^1(x^1) \le V_{*,\hpi}^1(x^1)$, and the second inequality applies Lemma \ref{lm:bellman_decomposition} with $c_t^h=0$.

Hence, for any $M\in\hcM$, we have with probability at least $1-\delta$,
\begin{align}\label{eqap:aan}
&\E_{\pi_*}\sum_{h=1}^H \tv(P_M^h(\cdot|z^h)\|P_{\bar M}^h(\cdot|z^h))\notag\\
&\qquad \le \E_{\pi_*}\sum_{h=1}^H \min\Bigg\{1,\frac{\tv\big(P_M^h(\cdot|z^h)\|P_{\bar M}^h(\cdot|z^h)\big)}{\sqrt{\lambda+\sum_{t=1}^T\tv\big(P_M^h(\cdot|z_t^h)\|P_{\bar M}^h(\cdot|z_t^h)\big)^2/\sigma_t^h}}\sqrt{\lambda+\sum_{t=1}^T \frac{\tv\big(P_M^h(\cdot|z_t^h)\|P_{\bar M}^h(\cdot|z_t^h)\big)^2}{\sigma_t^h}}\Bigg\}\notag\\
&\qquad \le 3\beta \E_{\pi_*}\sum_{h=1}^H \frac{\tv\big(P_M^h(\cdot|z^h)\|P_{\bar M}^h(\cdot|z^h)\big)}{\sqrt{\lambda+\sum_{t=1}^T\tv\big(P_M^h(\cdot|z_t^h)\|P_{\bar M}^h(\cdot|z_t^h)\big)^2/\sigma_t^h}}\notag\\
&\qquad = 3\beta \E_{\pi_*}\sum_{h=1}^H \bigg[\frac{\tv\big(P_M^h(\cdot|z^h)\|P_{\bar M}^h(\cdot|z^h)\big)}{\sqrt{\lambda+\sum_{t=1}^T\tv\big(P_M^h(\cdot|z_t^h)\|P_{\bar M}^h(\cdot|z_t^h)\big)^2/\sigma_t^h}}\mathbbm{1}(\sigma^h(z^h)=1)\notag\\
&\qquad\qquad\qquad + \frac{\tv\big(P_M^h(\cdot|z^h)\|P_{\bar M}^h(\cdot|z^h)\big)/\sigma^h(z^h)}{\sqrt{\lambda+\sum_{t=1}^T\tv\big(P_M^h(\cdot|z_t^h)\|P_{\bar M}^h(\cdot|z_t^h)\big)^2/\sigma_t^h}}\sigma^h(z^h)\mathbbm{1}(\sigma^h(z^h)>1)\bigg]\notag\\
&\qquad = 3\beta \E_{\pi_*}\sum_{h=1}^H \bigg[\sup_{M,M'\in\hcM}\frac{\tv\big(P_M^h(\cdot|z^h)\|P_{\bar M}^h(\cdot|z^h)\big)}{\sqrt{\lambda+\sum_{t=1}^T\tv\big(P_M^h(\cdot|z_t^h)\|P_{\bar M}^h(\cdot|z_t^h)\big)^2/\sigma_t^h}}\notag\\
&\qquad\qquad\qquad + \frac{1}{\alpha}\Big(\sup_{M,M'\in\hcM}\frac{\tv\big(P_M^h(\cdot|z^h)\|P_{\bar M}^h(\cdot|z^h)\big)}{\sqrt{\lambda+\sum_{t=1}^T\tv\big(P_M^h(\cdot|z_t^h)\|P_{\bar M}^h(\cdot|z_t^h)\big)^2/\sigma_t^h}}\Big)^2\bigg],
\end{align}
where the second inequality uses Lemma \ref{lm:off_Confidence Set}, and the last inequality holds due to the definition of $\sigma^h(z^h)$ when $\sigma^h(z^h)>1$:
\[
\sigma^h(z^h) = \frac{1}{\sigma^h(z^h)}\sup_{M,M'\in\hcM}\frac{\tv\big(P_M^h(\cdot|z^h)\|P_{M'}^h(\cdot|z^h)\big)/\alpha}{\sqrt{\lambda+\sum_{t=1}^T\tv\big(P_M^h(\cdot|z_t^h)\|P_{\bar M}^h(\cdot|z_t^h)\big)^2/\sigma_t^h}}.
\]
Further, by defining the weighted form of information coefficient
\[
\IC^{\sigma}(\lambda,\hcM,\cD) = \sup_{M,M'\in\hcM} \E_{\pi_*}\bigg[ \frac{T\cdot\tv(P_M^h(\cdot|z^h)\|P_{M'}^h(\cdot|z^h))/\sigma^h(z^h)^{1/2}}{\lambda + \sum_{t=1}^T\tv(P_M^h(\cdot|z_t^h)\|P_{M'}^h(\cdot|z_t^h))^2/\sigma_t^h} \bigg| x^1=x\bigg],
\]
we deduce that
\begin{align*}
\E_{\pi_*}\sum_{h=1}^H \tv(P_M^h(\cdot|z^h)\|P_{\bar M}^h(\cdot|z^h)) \le 3\beta H\bigg[\sqrt{\frac{\IC^{\sigma}(\lambda,\hcM,\cD)}{T}} + \frac{\IC^{\sigma}(\lambda,\hcM,\cD)}{\alpha T}\bigg].
\end{align*}
Then, we use the inequality above with $M=M_*$ and $\hM$ to get
\begin{align*}
\subopt(\hat{\pi},x^1) \le & 6\beta H\bigg[\sqrt{\frac{\IC^{\sigma}(\lambda,\hcM,\cD)}{T}} + \frac{\IC^{\sigma}(\lambda,\hcM,\cD)}{\alpha T}\bigg]\\
= & 6\big(5\sqrt{\log(|\cM|/\delta)\log^2 B} + 7\alpha C\big)H\bigg[\sqrt{\frac{\IC^{\sigma}(\lambda,\hcM,\cD)}{T}} + \frac{\IC^{\sigma}(\lambda,\hcM,\cD)}{\alpha T}\bigg]\\
= & \tO\Big(\frac{H\sqrt{\IC^{\sigma}(\lambda,\hcM,\cD)\log(|\cM|/\delta)\log^2 B}}{T} + \frac{\IC^{\sigma}(\lambda,\hcM,\cD)CH}{T}\Big),
\end{align*}
where we take $\alpha=\sqrt{\log|\cM|\log^2 B}/C$. Therefore, we complete the proof.
\end{proof}

\subsection{Proof of Theorem \ref{th:offline_upper}}\label{ss:Proof of Theorem offline_upper}
Then, we will analyze the relationship between the weighted information coefficient $\IC^{\sigma}(\lambda,\hcM,\cD)$ and $\Cov(\cM,\cD)$.

\begin{lemma}\label{lm:off Connections between Weighted and Unweighted Coefficient}
Under Assumption \ref{as:offline_upper}, we have
\$
\IC^{\sigma}(\lambda,\hcM,\cD) \le \frac{1}{\Cov(\cM,\cD)}.
\$
\end{lemma}
The proof is presented in Appendix \ref{ss:Lemmas for Offline Setting}.

Now, we are ready to prove Theorem \ref{th:offline_upper}.
\begin{proof}[Proof of Theorem \ref{th:offline_upper}]
According to Theorem \ref{th:Instance-Dependent Bound}, we have with probability at least $1-2\delta$,
\begin{align*}
\subopt(\hat{\pi},x^1) =  \tO\Big(\frac{H\sqrt{\IC^{\sigma}(\lambda,\hcM,\cD)\log(|\cM|/\delta)\log^2 B}}{T} + \frac{\IC^{\sigma}(\lambda,\hcM,\cD)CH}{T}\Big).
\end{align*}
By invoking Lemma \ref{lm:off Connections between Weighted and Unweighted Coefficient}, we complete the proof.
\end{proof}

\subsection{Offline Lower bound}\label{sec-offline-lower}
In this section, we provide the lower bound for offline learning.  Here, we used the hard-to-learn instance in Section \ref{ss:Proof of Theorem online_lower1}, while modified the transition probability at stage $h$ from $1/4$ or $3/4$ to
\$
&P_{a_h^*}^{h}(x_2|x_1,a)=
\begin{cases}
    \frac{1}{2}+\eta, & a = a_h^*\\
    \frac{1}{2}-\eta, & a \ne a_h^*,
\end{cases}\\
\$
for different corruption level $C$ and data-size $T$.

For the offline data collection process, we employ the behavior policy $\pi^v$, which with probability $1-\epsilon$ will select the action $a_d$, and with probability $\epsilon$, uniform select an action from $\{a_1,..,a_{d-1}\}$. For each stage $h\in[H]$, the optimal action $a^*_h$ is uniform randomly selected from the action set $\cA=\{a_1,...,a_{d-1}\}$ and we construct an auxiliary transition probability function $P^h_{0}$ without  optimal action such that
    \begin{align*}
        P^h_0(x_2|x_1,a)=\frac{1}{2}-\eta, & \forall a\in \cA.
    \end{align*}
Then the following lemma provides upper and lower bounds for the visiting times of state-action pair $(x_1,a)$.
    \begin{lemma}\label{lemma:offline-visit}
        For each stage $h\in[H]$ and fixed action $a\in \{a_1,...,a_{d-1}\}$, if the data-size $T$ satisfied $T>4e^2(d-1)^2H^2\log(1/\delta)/\epsilon^2$,then 
        with probability at least $1-\delta$, the behavior policy visit the state $x_1$ and take action $a$ no less than times $\epsilon T/ \big(4eH(d-1)\big)$ during the data collection process. In addition, with probability at least $1-\delta$, the visiting times is no more than $3\epsilon T/ \big(H(d-1)\big)$.
    \end{lemma}
\begin{proof}[Proof of Lemma \ref{lemma:offline-visit}]
    For any episode $T$ and stage $h\in[H]$, we define the random variable $y_i= \mathds{1} (x_{i}^h=x_1,a_i^h=a)$ as the indicator function of visiting the state $x_i$ and taking action $a$ at stage $h$ of episode $i$.
    Since the agent, starting from the current state $x_0$, will transition to state $x_1$ with a probability of $1/H$, regardless of the selected action, the probability that the agent visits state $x_1$ at stage $h$ and take action $a$ is upper bounded by
    \begin{align*}
        \Pr(y_i=1)=\bigg(1-\frac{1}{H}\bigg)^{h-1} \cdot \frac{1}{H}\cdot \frac{\epsilon}{d-1}\ge \frac{1}{eH}\cdot \frac{\epsilon}{d-1}.
    \end{align*}
    Therefore, applying the Azuma–Hoeffding inequality (Lemma \ref{lemma:azuma}), with a probability of at least $1-\delta$, we have:
    \begin{align*}
        \sum_{i=1}^T y_i\ge \frac{T}{eH}\cdot \frac{\epsilon}{d-1}- \sqrt{2T\log(1/\delta)}\ge \frac{T}{2eH}\cdot \frac{\epsilon}{d-1} - \frac{e(d-1)H}{\epsilon}\cdot\log (1/\delta),
    \end{align*}
    where the last inequality holds due to $x^2+y^2\ge 2xy$. Therefore, for $T>4e^2(d-1)^2H^2\log(1/\delta)/\epsilon^2$, with probability at least $1-\delta$, we have 
    \begin{align*}
        \sum_{i=1}^T y_i\ge \frac{T}{4eH} \cdot \frac{\epsilon}{d-1},
    \end{align*}
    which completes the proof of lower bounds in Lemma \ref{lemma:offline-visit}. 

    On the other hand, for the upper bound, we have 
    \begin{align}
        \Pr(y_i=1)=\bigg(1-\frac{1}{H}\bigg)^{h-1} \cdot \frac{1}{H} \leq \frac{1}{H}\cdot \frac{\epsilon}{d-1}.\notag
    \end{align}
    Similarly, applying the Azuma–Hoeffding inequality (Lemma \ref{lemma:azuma}), with a probability of at least $1-\delta$, we have:
    \begin{align*}
        \sum_{i=1}^T y_i\leq \frac{T}{H}\cdot \frac{\epsilon}{d-1}+ \sqrt{2T\log(1/\delta)}\leq \frac{2T}{H}\cdot \frac{\epsilon}{d-1}+\frac{(d-1)H}{2\epsilon}\cdot\log (1/\delta),
    \end{align*}
    where the last inequality holds due to $x^2+y^2\ge 2xy$. Therefore, for $T>4e^2(d-1)^2H^2\log(1/\delta)/\epsilon^2$, with probability at least $1-\delta$, we have 
    \begin{align*}
        \sum_{i=1}^T y_i\leq \frac{3T}{H} \cdot \frac{\epsilon}{d-1},
    \end{align*}
    which completes the proof of upper bounds in Lemma \ref{lemma:offline-visit}.
\end{proof}
For simplicity, we denote $\cE$ as the event where the high probability event in Lemma \ref{lemma:offline-visit} holds for all stage $h\in[H]$ and $a\in \{a_1,...,a_{d-1}\}$. Now, for behavior policy $\pi^{v}$, the following corruption strategy is applied to the transition probability $P_{a_h^*}^h$: if the optimal action $a^*_h$ is selected at stage $h$ with the current state $x_1$, then the adversary corrupts the transition probability $P_{a_h^*}^h$ to $P^h_0$.
Under this situation, the following lemma provides the upper bound of corruption level and lower bounds for the data coverage coefficient $\Cov(\cM,\cD)$.
\begin{lemma}\label{lemma:offline-constant}
  Conditioned on the event $\cE$, the corruption level is upper bounded by 
  \begin{align*}
      C\leq \frac{3\epsilon T}{(d-1)H} \cdot\frac{4\eta}{1-2\eta}.
  \end{align*}
  In addition, the data coverage coefficient $\Cov(\cM,\cD)$ is lower bounded by
  \begin{align*}
      \Cov(\cM,\cD)\ge \frac{\epsilon}{4eH(d-1)}.
  \end{align*}
\end{lemma}
\begin{proof}[Proof of Lemma \ref{lemma:offline-constant}]
According the definition of corruption strategy, if the optimal action $a^*_h$ is selected at stage $h$ with the current state $x_1$, the corresponding corruption level within this episode is denoted by
\begin{align}
    c_t^h = c_t^h(x_1,a_h^*) = \sup_{x^{h+1}\in\Delta_t(\cX)}\Big|\frac{P_0^h(x^{h+1}|x_1,a_h^*)}{P_{a_h^*}^j(x^{h+1}|x_1,a_h^*)}-1\Big|=\frac{4\eta}{1-2\eta}.\notag
\end{align}
It is worth noting that the event $\cE$ only focused on the transition before reaching the state $x_1$, and remains unaffected by the adversary corruption employed. Therefore, conditioned on the event $\cE$, for each stage $h\in[H]$, the total corruption level throughout the offline data collection process is upper-bounded by
\begin{align*}
    \sum_{i=1}^T c_t^h\leq \frac{3\epsilon T}{(d-1)H}\cdot\frac{4\eta}{1-2\eta}.
\end{align*}
Thus, the corruption level for the employed strategy satisfied
\begin{align*}
    C\leq \frac{3\epsilon T}{(d-1)H}\cdot\frac{4\eta}{1-2\eta}.
\end{align*}
Regarding the data coverage coefficient $\Cov(\cM,\cD)$, it is worth noting that all models $M \in \cM$ share identical transition probabilities across states ${x_0,x_2,x_3}$. Consequently, we only need to focus on state $x_1$ and action $a\in \{a_1,...,a_{d-1}\}$. Conditioned on event $\cE$, we have
\begin{align*}
    \rho^h(M,M')^2&\leq \sum_{a\in \{a_1,...,a_{d-1}\}}\tv(P^h_{M}(\cdot|x_1,a)\|P^h_{M'}(\cdot|x_1,a))^2\\ 
    &\leq \frac{4eH(d-1)}{\epsilon}\cdot \frac{1}{T}\sum_{t=1}^T\tv(P_{M}^h(\cdot|z_t^h)\|P_{M'}^h(\cdot|z_t^h))^2,
\end{align*}
where the second inequality holds due to the definition of events $\cE$. Thus, we have
\begin{align*}
    \Cov(\cM,\cD)\ge \frac{\epsilon}{4eH(d-1)},
\end{align*}
and we complete the proof of Lemma \ref{lemma:offline-constant}.
\end{proof}
With the help of these lemmas, we can start the proof of Theorem \ref{th:offline_lower}.
\begin{proof}[Proof of Theorem \ref{th:offline_lower}]
According to the corruption strategy, regardless of the actual optimal action $a_h^*$, during the offline collection process, the behavior of the transition probability function is same as $P_0^h$. Rough speaking, the agent cannot outperform random guessing of the optimal action $a_h^*$ and subsequently outputting the corresponding optimal policy $\hat{\pi}^h$. Thus, when the optimal action $a^*_h$ is uniform randomly selected from the action set $\cA=\{a_1,...,a_{d-1}\}$, the sub-optimally gap of policy $\hat{\pi}$ can be denoted by
\begin{align*}
    \mathbb{E}[\subopt(\hat\pi, x)]&=\mathbb{E}\Big[\sum_{h=2}^{H-1} \mathds{1} (x^h=x_1)\cdot \mathds{1}  (\hat{\pi}^h(x_1) \ne a_h^*) \cdot 2\eta\Big]\notag\\
    &\ge \mathbb{E}\Big[\sum_{h=2}^{H-1} \frac{1}{eH}\cdot \mathds{1}  (\hat{\pi}^h(x_1) \ne a_h^*) \cdot 2\eta\Big]\notag\\
    &=\frac{H-1}{eH}\cdot \frac{d-2}{d-1}\cdot 2\eta.
\end{align*}
Now, for a given dataset size $T$, corruption level $C$ and data coverage coefficient $\Cov(\cM,\cD)$, according to Lemma \ref{lemma:offline-constant}, we can select the parameter as following:
\begin{align*}
    \epsilon=\Cov(\cM,\cD)\cdot 4eH(d-1),\eta =\frac{\Cov(d-1)H}{24 \epsilon T}=\frac{C}{96e \Cov(\cM,\cD) T}.
\end{align*}
Then, if the dataset size $T$ satisfied $T>C/(24e\Cov(\cM,\cD))$ and $d>3,H>2$, then we have
\begin{align*}
    \mathbb{E}[\subopt(\hat\pi, x)]\ge \Omega(\eta)= \Omega\Big(\frac{C}{\Cov(\cM,\cD) T}\Big). 
\end{align*}
Thus, we complete the proof of Theorem \ref{th:offline_lower}.
\end{proof}





\section{Proof of Supporting Lemmas}
\subsection{Lemmas for Online Setting}\label{ss:Lemmas for Online Setting}
\begin{proof}[Proof of Lemma \ref{lm:confidence_set}]
For simplicity, we assume that class $\cM$ has finite elements.
Let
\$
    E_t = \sum_{s=1}^t\tv\big(P_*^h(\cdot|z_s^h)\|P_{\bar M_{t+1}}^h(\cdot|z_s^h)\big)^2/\sigma_s^h.
\$
For each time $s\in[t-1]$, we define the region $\cD_s^h=\{x^{h+1}\in\cX:~P_s^h(x^{h+1}|z_s^h)\le P_*^h(x^{h+1}|z_s^h)\}$. For any fixed $M\in\cM$, we have
\#\label{eq:ap_aaa}
    &\E\sum_{s=1}^{t-1}\frac{1}{\sigma_s^h}\log\sqrt{\frac{\md P^h_{M}(x^{h+1}|z_s^h)}{\md P^h_*(x^{h+1}|z_s^h)}}\notag\\
    &\qquad = \underbrace{\sum_{s=1}^{t-1}\frac{1}{\sigma_s^h}\int_{\cD_s^h}\md P^h_s(x^{h+1}|z_s^h)\log\sqrt{\frac{\md P^h_{M}(x^{h+1}|z_s^h)}{\md P^h_*(x^{h+1}|z_s^h)}}}_{P_1} + \underbrace{\sum_{s=1}^{t-1}\frac{1}{\sigma_s^h}\int_{\bar \cD_s^h}\md P^h_s(x^{h+1}|z_s^h)\log\sqrt{\frac{\md P^h_{M}(x^{h+1}|z_s^h)}{\md P^h_*(x^{h+1}|z_s^h)}}}_{P_2}.
\#
For the term $P_1$, we derive
\#\label{eq:ap_aab}
    P_1 &\le \frac{1}{2}\sum_{s=1}^{t-1}\frac{1}{\sigma_s^h}\int_{\cD_s^h}\md P^h_*(x^{h+1}|z_s^h)\log\Big(\frac{\md P^h_{M}(x^{h+1}|z_s^h)}{\md P^h_*(x^{h+1}|z_s^h)}\Big)\notag\\
    &=-\frac{1}{2}\sum_{s=1}^{t-1}\frac{1}{\sigma_s^h}\int_{\cD_s^h}\md P^h_*(x^{h+1}|z_s^h)\log\Big(\frac{\md P^h_*(x^{h+1}|z_s^h)}{\md P^h_{M}(x^{h+1}|z_s^h)}\Big).
\#
For the term $P_2$, we have
\#\label{eq:ap_aac}
    P_2 &= \sum_{s=1}^{t-1}\underbrace{\frac{1}{2\sigma_s^h}\int_{\bar \cD_s^h}\big(\md P^h_s(x^{h+1}|z_s^h)-\md P^h_*(x^{h+1}|z_s^h)\big)\log\Big(\frac{\md P^h_{M}(x^{h+1}|z_s^h)}{\md P^h_*(x^{h+1}|z_s^h)}\Big)}_{Q}\notag\\
    &\qquad- \sum_{s=1}^{t-1}\frac{1}{2\sigma_s^h}\int_{\bar \cD_s^h}\md P^h_*(x^{h+1}|z_s^h)\log\Big(\frac{\md P^h_*(x^{h+1}|z_s^h)}{\md P^h_{M}(x^{h+1}|z_s^h)}\Big),
\#
from which we further bound the term $Q$ by
\#\label{eq:ap_aad}
    Q &\le \frac{1}{2\sigma_s^h}\int_{\bar \cD_s^h}\big(\md P^h_s(x^{h+1}|z_s^h)-\md P^h_*(x^{h+1}|z_s^h)\big)\cdot\bigg(\frac{\md P^h_{M}(x^{h+1}|z_s^h)}{\md P^h_*(x^{h+1}|z_s^h)}-1\bigg)\notag\\
    &\le \frac{1}{2\sigma_s^h}\sup_{x^{h+1}\in\bar \cD_s^h}\Big|\frac{\md P^h_s(x^{h+1}|z_s^h)}{\md P^h_*(x^{h+1}|z_s^h)}-1\Big|\int_{\bar \cD_s^h}\Big|\md P^h_{M}(x^{h+1}|z_s^h)-\md P^h_*(x^{h+1}|z_s^h)\Big|\notag\\
    &\le \frac{c_s^h}{2}\cdot\frac{\tv\big(P^h_*(\cdot|z_s^h)\big\|P^h_{M}(\cdot|z_s^h)\big)}{\sigma_s^h}.
\#
Therefore, by combining \eqref{eq:ap_aab}, \eqref{eq:ap_aac} and \eqref{eq:ap_aad} and take them back into \eqref{eq:ap_aaa}, we obtain that
\$
    &\E\sum_{s=1}^{t-1}\frac{1}{\sigma_s^h}\log\sqrt{\frac{\md P^h_{M}(x^{h+1}|z_s^h)}{\md P^h_*(x^{h+1}|z_s^h)}}\notag\\
    &\qquad\le -\frac{1}{2}\sum_{s=1}^{t-1}\frac{1}{\sigma_s^h}H\big(P^h_*(\cdot|z_s^h)\big\|P^h_{M}(\cdot|z_s^h)\big)^2 + \frac{1}{2}\sum_{s=1}^{t-1}\frac{c_s^h\tv\big(P^h_*(\cdot|z_s^h)\big\|P^h_{M}(\cdot|z_s^h)\big)}{\sigma_s^h}.
\$
Then, we need to bound the Variance:
\#\label{eq:ap_aae}
    \Var\bigg[\frac{1}{\sigma_s^h}\log\sqrt{\frac{\md P^h_{M}(x^{h+1}|z_s^h)}{\md P^h_*(x^{h+1}|z_s^h)}}\bigg] \le& \frac{1}{(\sigma_s^h)^2}\E\bigg[\log^2\sqrt{\frac{\md P^h_{M}(x^{h+1}|z_s^h)}{\md P^h_*(x^{h+1}|z_s^h)}}\bigg]\notag\\
    \le& \underbrace{\frac{1}{(\sigma_s^h)^2}\E_{x^{h+1}\sim|P_s^h(\cdot|z_s^h)-P_*^h(\cdot|z_s^h)|}\bigg[\log^2\sqrt{\frac{\md P^h_{M}(x^{h+1}|z_s^h)}{\md P^h_*(x^{h+1}|z_s^h)}}\bigg]}_{Q_1}\notag\\
    &\qquad + \underbrace{\frac{1}{(\sigma_s^h)^2}\E_{x^{h+1}\sim P_*^h(\cdot|z_s^h)}\bigg[\log^2\sqrt{\frac{\md P^h_{M}(x^{h+1}|z_s^h)}{\md P^h_*(x^{h+1}|z_s^h)}}\bigg]}_{Q_2}.
\#
For the term $Q_2$, by invoking Lemma \ref{lm:divergence inequality} and using Assumption \ref{as:Bounded Condition}, we obtain
\#\label{eq:ap_aaf}
    Q_2 &= \frac{1}{4(\sigma_s^h)^2}\E_{x^{h+1}\sim P_*^h(\cdot|z_s^h)}\bigg[\log^2\bigg(\frac{\md P^h_*(x^{h+1}|z_s^h)}{\md P^h_{M}(x^{h+1}|z_s^h)}\bigg)\bigg]\notag\\
    &\le \frac{1+\log B}{2(\sigma_s^h)^2}\kl\big(P^h_*(\cdot|z_s^h)\big\|P^h_{M}(\cdot|z_s^h)\big)\notag\\
    &\le \frac{(1+\log B)(3+\log B)}{2(\sigma_s^h)^2}H\big(P^h_*(\cdot|z_s^h)\big\|P^h_{M}(\cdot|z_s^h)\big)^2.
\#
For the term $Q_1$,
\#\label{eq:ap_aag}
    Q_1 &= \frac{1}{4(\sigma_s^h)^2}\int\big|\md P^h_s(x^{h+1}|z_s^h)-\md P^h_*(x^{h+1}|z_s^h)\big|\cdot\log^2\bigg(\frac{\md P^h_*(x^{h+1}|z_s^h)}{\md P^h_{M}(x^{h+1}|z_s^h)}\bigg)\notag\\
    &= \frac{1}{4(\sigma_s^h)^2}\sup_{x^{h+1}\in\cX}\Big|\frac{\md P^h_s(x^{h+1}|z_s^h)}{\md P^h_*(x^{h+1}|z_s^h)}-1\Big|\int\md P^h_*(x^{h+1}|z_s^h)\log^2\bigg(\frac{\md P^h_*(x^{h+1}|z_s^h)}{\md P^h_{M}(x^{h+1}|z_s^h)}\bigg)\notag\\
    &\le \frac{(\log B+1)c_s^h}{2(\sigma_s^h)^2}\cdot\kl\big(P^h_*(\cdot|z_s^h)\big\|P^h_{M}(\cdot|z_s^h)\big)\notag\\
    &\le \frac{(\log B+1)(\log B+3)c_s^h}{2(\sigma_s^h)^2}\cdot H\big(P^h_*(\cdot|z_s^h)\big\|P^h_{M}(\cdot|z_s^h)\big)^2\notag\\
    &\le \frac{(\log B+1)(\log B+3)c_s^h}{(\sigma_s^h)^2}\cdot \tv\big(P^h_*(\cdot|z_s^h)\big\|P^h_{M}(\cdot|z_s^h)\big),
\#
where the first inequality invokes Lemma \ref{lm:divergence inequality}, the second inequality uses Lemma \ref{lm:kl and Hellinger}, and the last inequality holds due to $H(P\|Q)^2\le2\tv(P\|Q)$.
Therefore, from \eqref{eq:ap_aae}, \eqref{eq:ap_aaf} and \eqref{eq:ap_aag}, the Variance is bounded as
\$
    \Var\bigg[\frac{1}{\sigma_s^h}\log\sqrt{\frac{\md P^h_{M}(x^{h+1}|z_s^h)}{\md P^h_*(x^{h+1}|z_s^h)}}\bigg] &\le \frac{(\log B+1)(\log B+3)}{2(\sigma_s^h)^2}H\big(P^h_*(\cdot|z_s^h)\big\|P^h_{M}(\cdot|z_s^h)\big)^2\notag\\
    &\qquad + \frac{(\log B+1)(\log B+3)c_s^h}{(\sigma_s^h)^2}\cdot \tv\big(P^h_*(\cdot|z_s^h)\big\|P^h_{M}(\cdot|z_s^h)\big).
\$
By applying Lemma \ref{lm:freedman} with $\lambda_0< 3/\log B$ and $b=\log B$, we get with probability at least $1-\delta$, for any $M\in\cM$
\$
&\sum_{s=1}^{t-1}\frac{1}{\sigma_s^h}\log\sqrt{\frac{\md P^h_{M}(x^{h+1}|z_s^h)}{\md P^h_*(x^{h+1}|z_s^h)}}\\
&\qquad\le  \frac{\log(|\cM|/\delta)}{\lambda_0} -\frac{1}{2}\sum_{s=1}^{t-1}\frac{1}{\sigma_s^h}H\big(P^h_*(\cdot|z_s^h)\big\|P^h_{M}(\cdot|z_s^h)\big)^2 + \frac{1}{2}\sum_{s=1}^{t-1}\frac{c_s^h\tv\big(P^h_*(\cdot|z_s^h)\big\|P^h_{M}(\cdot|z_s^h)\big)}{\sigma_s^h}\\
&\qquad\qquad+ \frac{\lambda_0(\log B+1)(\log B+3)}{2(1-\lambda_0\log 
B/3)}\sum_{s=1}^{t-1}\bigg(\frac{1}{2(\sigma_s^h)^2}H\big(P^h_*(\cdot|z_s^h)\big\|P^h_{M}(\cdot|z_s^h)\big)^2 + \frac{c_s^h}{(\sigma_s^h)^2}\tv\big(P^h_*(\cdot|z_s^h)\big\|P^h_{M}(\cdot|z_s^h)\big)\bigg).
\$
By taking $\lambda_0= 3/(19\log^2 B)$ and $M=\bar M_t$, we further get
\#\label{eq:ap_aah}
&\sum_{s=1}^{t-1}\frac{1}{\sigma_s^h}\log\sqrt{\frac{\md P^h_{\bar M_t}(x^{h+1}|z_s^h)}{\md P^h_*(x^{h+1}|z_s^h)}}\notag\\
&\qquad\le \frac{19\log(|\cM|/\delta)\log^2 B}{3} -\frac{1}{4}\sum_{s=1}^{t-1}\frac{1}{\sigma_s^h} H\big(P^h_*(\cdot|z_s^h)\big\|P^h_{\bar M_t}(\cdot|z_s^h)\big)^2\notag\\
&\qquad\qquad + \sum_{s=1}^{t-1}\frac{c_s^h\big(\tv\big(P^h_*(\cdot|z_s^h)\big\|P^h_{\bar M_s}(\cdot|z_s^h)\big) + \tv\big(P^h_{\bar M_s}(\cdot|z_s^h)\big\|P^h_{\bar M_t}(\cdot|z_s^h)\big)\big)}{\sigma_s^h}\notag\\
&\qquad\le \frac{19\log(|\cM|/\delta)\log^2 B}{3} -\frac{1}{4}\sum_{s=1}^{t-1}\frac{1}{\sigma_s^h} \tv\big(P^h_*(\cdot|z_s^h)\big\|P^h_{\bar M_t}(\cdot|z_s^h)\big)^2 + 2\alpha C\max_s\sqrt{\lambda+E_{s-1}},
\#
where the last inequality uses the induction that $M_*\in\cM_s$ and $\bar M_t\in\cM_{t-1}\subseteq\cM_s$, and the definition of the weight:
\begin{align*}
\sigma_s^h \ge & \frac{1}{\alpha}\cdot\sup_{M\in\cM_s}\frac{\tv\big(P_{M}^h(\cdot|z_s^h)\|P_{\bar M_s}^h(\cdot|z_s^h)\big)}{\sqrt{\lambda+\sum_{\tau=1}^{s-1}\tv\big(P_{M}^h(\cdot|z_{\tau}^h)\|P_{\bar M_s}^h(\cdot|z_{\tau}^h)\big)^2/\sigma_{\tau}^h}}\\
\ge& \frac{1}{\alpha}\cdot\frac{\tv\big(P_{M_*}^h(\cdot|z_s^h)\|P_{\bar M_s}^h(\cdot|z_s^h)\big)}{\sqrt{\lambda+\sum_{\tau=1}^{s-1}\tv\big(P_{M_*}^h(\cdot|z_{\tau}^h)\|P_{\bar M_s}^h(\cdot|z_{\tau}^h)\big)^2/\sigma_{\tau}^h}}.
\end{align*}
Since $\bar M_t$ is the maximizer of the log-likelihood,
\$
\sum_{s=1}^{t-1}\frac{1}{\sigma_s^h}\tv\big(P^h_*(\cdot|z_s^h)\big\|P^h_{\bar M_t}(\cdot|z_s^h)\big)^2 \le \frac{76\log(|\cM|/\delta)\log^2 B}{3} + 8\alpha C\max_s\sqrt{\lambda+E_{s-1}}.
\$

Additionally, we use the non-negativity of TV distance to get from \eqref{eq:ap_aah} that
\$
\sum_{s=1}^{t-1} \frac{1}{\sigma_s^h}\log P_*^h(x_s^{h+1}|z_s^h) &\ge \sum_{s=1}^t \frac{1}{\sigma_s^h}\log P_{\bar{M}_t}^h(x_s^{h+1}|z_s^h) - \frac{38\log(|\cM|/\delta)\log^2 B}{3} - 4\alpha C\max_s\sqrt{\lambda+E_{s-1}},
\$
which implies $M_*\in\cM_t$.

Moreover, for any $M\in\cM_t$ satisfying that
\begin{align}\label{eq:ap_aao}
\sum_{s=1}^{t-1} \frac{1}{\sigma_s^h}\log P_M^h(x_s^{h+1}|z_s^h) \ge \sum_{s=1}^t \frac{1}{\sigma_s^h}\log P_{\bar{M}_t}^h(x_s^{h+1}|z_s^h) - \beta^2,    
\end{align}
by taking this back into \eqref{eq:ap_aah} with a general $M$, we have
\begin{align*}
    &\frac{\beta^2}{2} -\frac{1}{4}\sum_{s=1}^{t-1}\frac{1}{\sigma_s^h} \tv\big(P^h_*(\cdot|z_s^h)\big\|P^h_{\bar M_t}(\cdot|z_s^h)\big)^2\\
    &\qquad \ge \sum_{s=1}^{t-1}\frac{1}{\sigma_s^h}\log\sqrt{\frac{\md P^h_{M}(x^{h+1}|z_s^h)}{\md P^h_*(x^{h+1}|z_s^h)}}\\
    &\qquad = \sum_{s=1}^{t-1}\frac{1}{\sigma_s^h}\log\sqrt{\frac{\md P^h_{M}(x^{h+1}|z_s^h)}{\md P^h_{\bar M_t}(x^{h+1}|z_s^h)}} + \sum_{s=1}^{t-1}\frac{1}{\sigma_s^h}\log\sqrt{\frac{\md P^h_{\bar M_t}(x^{h+1}|z_s^h)}{\md P^h_*(x^{h+1}|z_s^h)}}\\
    &\qquad \ge \sum_{s=1}^{t-1}\frac{1}{\sigma_s^h}\log\sqrt{\frac{\md P^h_{M}(x^{h+1}|z_s^h)}{\md P^h_{\bar M_t}(x^{h+1}|z_s^h)}} \ge -\frac{\beta^2}{2},
\end{align*}
which indicates that
\[
\sum_{s=1}^{t-1}\frac{1}{\sigma_s^h} \tv\big(P^h_*(\cdot|z_s^h)\big\|P^h_{\bar M_t}(\cdot|z_s^h)\big)^2 \le 4\beta^2.
\]
Therefore, we complete the proof.
\end{proof}

\begin{proof}[Proof of Lemma \ref{lm:bellman_decomposition}]
Start at step $h=1$:
\$
&V_M^1(x^1)-V_{\pi_M}^1(x^1)\\
&\qquad = \E_{a^1\sim\pi_M}\big[\E^M V_M^2(x^2) - \E^* V_{\pi_M}^2(x^2)\big]\\
&\qquad = \E_{a^1\sim\pi_M}\big[(\E^M V_M^2(x^2) - \E^* V_M^2(x^2)) + (\E^* V_M^2(x^2) - \E^t V_M^2(x^2)) + (\E^t V_M^2(x^2) - \E^t V_{\pi_M}^2(x^2))\\
&\qquad\quad + (\E^t V_M^2(x^2) V_{\pi_M}^2(x^2) - \E^* V_{\pi_M}^2(x^2))\big]\\
&\qquad \le \E_{a^1\sim\pi_M}\big[\E^M V_M^2(x^2) - \E^* V_M^2(x^2) + 2\tv(P_t^h(\cdot|x^1,a^1)\|P_*^h(\cdot|x^1,a^1)) + \E^t[V_M^2(x^2) - V_{\pi_M}^2(x^2)]\big]\\
&\qquad \le \E_{a^1\sim\pi_M}\big[\cE^1(M,x^1,a^1) + c_t^h(x^1,a^1) + \E^t[V_M^2(x^2) - V_{\pi_M}^2(x^2)]\big],
\$
where the first inequality holds from the fact that $V_M(\cdot),V_{\pi_M}(\cdot)\in[0,1]$, and the second inequality is due to the definition of bellman error and $\tv(P\|Q)\le 1/2\cdot\sup_x|P(x)/Q(x)-1|$. By further expanding the last term on the right-hand side of the inequality above, we complete the first inequality of the Lemma. Similarly, we can obtain the second inequality.
\end{proof}

\begin{proof}[Proof of Lemma \ref{lm:Relation between Information Ratio and Eluder Dimension}]
To condense notations, we use the notation $l(M,\bar M_t,z)=\tv(P_M(\cdot|z)\|P_{\bar M_t}(\cdot|z))$. Now, we follow the three steps in the proof of Lemma 5.1 from \citet{ye2023corruptiona}. 
\paragraph{Step I: Matched levels} The first step is to divide the sample set $\cS_T$ into $\log_2(T/\lambda)+1$ disjoint subsets
\$
\cS_T = \cup_{\iota=0}^{\log_2(T/\lambda)}\cS^{\iota}.
\$
For each $z_t\in\cS_T$, let $M_{z_t}\in\cM_t$ be the maximizer of
\$
\frac{l(M_{z_t}, \bar M_t, z_t)/\sigma_t^{1/2}}{\sqrt{\lambda + \sum_{s=1}^{t-1}l(M_{z_t}, \bar M_t, z_s)^2/\sigma_s}}.
\$
Since $l(M_{z_t}, \bar M_t, z_t)\in[0,1]$, we can decompose $\cS_T$ into $\log_2(T/\lambda)$ disjoint subsequences:
\[
\cS^{\iota} = \{z_t\in\cS_T \,|\, l(M_{z_t}, \bar M_t, z_t)^2\in(2^{-\iota-1},2^{-\iota}]\},
\]
and
\[
\cS^{\log_2(T/\lambda)} = \{z_t\in\cS_T \,|\, l(M_{z_t}, \bar M_t, z_t)^2\in[0,\lambda/T]\}.
\]
Correspondingly, we also divide $\rR^+$ into $\log_2(T/\lambda)+1$ disjoint subsets:
\[
\rR^+ = \cup_{\iota=-1}^{\log_2(T/\lambda)}\cR^{\iota},
\]
where we define 
\$
&\cR^{\iota}=[2^{\iota/2}\log|\cM|, 2^{(\iota+1)/2}\log|\cM|),~\text{for}~\iota=0,\ldots,\log_2(T/\lambda)-1,\\
&\cR^{\log_2(T/\lambda)}=[\sqrt{T/\lambda}\log|\cM|,+\infty),~ \cR^{\log_2(T/\lambda)} = [0,\log|\cM|).
\$
Then, there exists an $\iota_0\in\{-1,0,\ldots,\log_2(T/\lambda)\}$ such that $C\in\cR^{\iota_0}$.

\paragraph{Step II: Control weights in each level} For any $z_t\in\cS^{\log_2(T/\lambda)}$, we have
\#\label{eq:D_lambda_F_t_1}
\big(I_{\sigma}(\lambda,\cM,\cS_t)\big)^2 \le & \sup_{M\in\cM_t}\frac{l(M,\bar{M}_t,z_t)^2/\sigma_t^2}{\lambda + \sum_{s=1}^{t-1}l(M_{z_t},\bar M_t,z_s)^2/\sigma_s}\notag\\
\le & \frac{l(M,\bar{M}_t,z_t)^2}{\lambda} \le \frac{1}{T},
\#
which implies that
\[
\sum_{z_t\in\cS^{\log_2(T/\lambda)}}\big(I_{\sigma}(\lambda,\cM,\cS_t)\big)^2 \le 1.
\]
Moreover, for $\iota=0,\ldots,\log_2(T/\lambda)-1$, we aim to control the upper and lower bound of weights $\{\sigma_t:z_t\in\cS^{\iota}\}$. We define for $t\in[T]$,
\$
\psi_t =& \frac{1}{\alpha}\sup_{M\in\cM_t}\frac{l(M, \bar M_t, z_t)}{\sqrt{\lambda + \sum_{s=1}^{t-1}l(M, \bar M_t, z_s)^2/\sigma_s}}\\
= & \frac{1}{\alpha}\frac{l(M_{z_t}, \bar M_t, z_t)}{\sqrt{\lambda + \sum_{s=1}^{t-1}l(M_{z_t}, \bar M_t, z_s)^2/\sigma_s}}.
\$

When $\iota>\iota_0$, we have
\[
C < 2^{(\iota_0+1)/2}\log|\cM| < 2^{\iota/2}\log|\cM|.
\]
For any $z_t\in\cS^{\iota}$, we know that $l(M_{z_t},\bar M_t, z_t)^2<2^{-\iota}$. Hence, it follows that
\[
\psi_t \le \frac{1}{\alpha}\cdot\frac{2^{-\iota/2}}{\sqrt{\lambda}} \le \frac{C}{\sqrt{\log|\cM|}}\cdot\frac{2^{-\iota/2}}{\sqrt{\log|\cM|}} \le 1
\]
Since $\sigma_t=\max\{1,\psi_t\}$, we get $\sigma_t=1$ for all $z_t\in\cS^{\iota}$. 

When $\iota\le\iota_0$, for all $z_t\in\cS^{\iota}$ we get $C\ge2^{\iota_0/2}\log|\cM|\ge2^{\iota/2}\log|\cM|$, and $l(M,\bar M_t,z_t)^2\in(2^{-\iota-1},2^{-\iota})$. Then, we can verify that
\$
&\psi_t \le \frac{C}{\sqrt{\log|\cM|}}\cdot\frac{2^{-\iota/2}}{\sqrt{\log|\cM|}} = \frac{C}{2^{\iota/2}\log|\cM|},\\
&\psi_t \ge \frac{C}{\sqrt{\log|\cM|}}\cdot\frac{2^{-(\iota+1)/2}}{\sqrt{c_0\log|\cM|}} = \frac{C}{\sqrt{2c_0}2^{\iota/2}\log|\cM|},
\$
where the inequality of the second row applies
\$
\lambda + \sum_{s=1}^{t-1}l(M,\bar M_t,z_t) \le \lambda + \beta^2 \le c_0\log|\cM|.
\$
Since $C/(2^{\iota/2}\log|\cM|)\ge 1$, we further have for all $z_t\in\cS^{\iota}$
\$
\sigma_t^2 \in \Big[\frac{C}{\sqrt{2c_0}2^{\iota/2}\log|\cM|},\frac{C}{2^{\iota/2}\log|\cM|}\Big].
\$

\paragraph{Step III: Bound the sum} In this step, we bound the sum $\sum_{z_t\in\cS^{\iota}}(I_{\sigma}(\lambda,\cM,\cS_t))^2$ for each $\iota=0,\ldots,\log_2(T/\lambda)-1$. Fixing an $\iota$, we can decompose $\cS^{\iota}$ into $N^{\iota}+1$ disjoint subsets:
\$
\cS^{\iota} = \cup_{j=1}^{N^{\iota}+1}\cS_j^{\iota},
\$
where we define $N^{\iota}=|\cS^{\iota}|/\ED(\cM,2^{(-\iota-1)/2})$. With a slight abuse of notation, we have $\cS^{\iota}=\{z_i\}_{i=1}^{|\cS^{\iota}|}$, where the elements are arranged in the same order as in the original set $\cS_T$. Initially, let $\cS_j^{\iota}=\{\}$ for all $j\in[N^{\iota}+1]$. 
From $i=1$ to $|\cS^{\iota}|$, we find the smallest $j\in[N^{\iota}]$ such that $z_i$ is $2^{(-\iota-1)/2}$-independent of $\cS_j^{\iota}$ with respect to $\cM$. If such a $j$ does not exist, set $j=N^{\iota}+1$. Then, let the choice of $j$ for each $z_i$ be $j(z_i)$. 
According to the design of the procedure, it is obvious that for all $z_i\in\cS^{\iota}$, $z_i$ is $2^{(-\iota-1)/2}$-dependent on each of $\cS_{1,i}^{\iota},\ldots,\cS_{j(z_i)-1,i}^{\iota}$, where $\cS_{k,i}^{\iota}=\cS_{t}^{\iota}\cap\{z_1,\ldots,z_{i-1}\}$ for $k=1,\ldots,j(z_i)-1$.

For any $z_i\in\cS^{\iota}$ indexed by $t$ in $\cS_T$, 
we have $l(M_{z_t},\bar M_t,z_t)^2\ge 2^{-\iota-1}$. Then, because $z_i$ is $2^{(-\iota-1)/2}$-dependent on $\cS_{1,i}^{\iota},\ldots,\cS_{j(z_i)-1,i}^{\iota}$, respectively, we get for each $k=1,\ldots,j(z_i)-1$,
$$
\sum_{z\in\cS_{k,i}^{\iota}}l(M_{z_t},\bar M_t,z)^2 \ge 2^{-\iota-1}.
$$
Then, we obtain
$$
\frac{l(M_{z_t},\bar M_t,z_t)^2/\sigma_t}{\lambda+\sum_{s=1}^{t-1}l(M_{z_t},\bar M_t,z_s)^2/\sigma_s} \le \frac{2^{-\iota}/\sigma_t}{\lambda+\sum_{k=1}^{j(z_i)-1}\sum_{z_s\in\cS_{k,i}^{\iota}}l(M_{z_t},\bar M_t,z)^2/\sigma_s}.
$$

When $\iota>\iota_0$, recall from step II that $\sigma_t=1$ for all $z_t\in\cS^{\iota}$. Thus, we get
$$
\frac{l(M_{z_t},\bar M_t,z_t)^2/\sigma_t}{\lambda+\sum_{s=1}^{t-1}l(M_{z_t},\bar M_t,z_s)^2/\sigma_s} \le \frac{2^{-\iota}}{\lambda + (j(z_i)-1)2^{-\iota-1}} = \frac{2}{j(z_i)-1 + \lambda2^{\iota+1}}.
$$
By summing over all $z_t\in\cS^{\iota}$, we obtain
\#\label{eq:D_lambda_F_t_2}
\sum_{z_t\in\cS^{\iota}} (I_{\sigma}(\lambda,\cM,\cS_t))^2 &\le \sum_{j=1}^{N^{\iota}}\sum_{z_i\in\cS_j^{\iota}} \frac{2}{j-1 + \lambda2^{\iota+1}} + \sum_{z_i\in\cS_{N^{\iota}+1}^{\iota}} \frac{2}{N^{\iota}}\nonumber\\
&\le \sum_{j=1}^{N^{\iota}} \frac{2|\cS_j^{\iota}|}{j} + \frac{2|\cS_{N^{\iota}+1}^{\iota}|}{N^{\iota}}\nonumber\\
&\le 2\ED(\cM,2^{(-\iota-1)/2})\log N^{\iota} + 2|\cS^{\iota}|\cdot\frac{\ED(\cM,2^{(-\iota-1)/2})}{|\cS^{\iota}|}\nonumber\\
&\le 4\ED(\cM,2^{(-\iota-1)/2})\log N^{\iota},
\#
where the third inequality is deduced since by the definition of eluder dimension, we have $|\cS_j^{\iota}|\le\ED(\cM,2^{(-\iota-1)/2})$ for all $j\in[N^{\iota}]$.

When $\iota\le\iota_0$, we have from step II that $\sigma_t^2\in[C/(\sqrt{2c_0}2^{\iota/2}\log N), C/(2^{\iota/2}\log N)]$ for all $z_t\in\cS^{\iota}$, which indicates that their weights are roughly of the same order. Then, we obtain that
\$
\frac{l(M_{z_t},\bar M_t,z_t)^2/\sigma_t}{\lambda+\sum_{s=1}^{t-1}l(M_{z_t},\bar M_t,z_s)^2/\sigma_s} &\le \frac{l(M_{z_t},\bar M_t,z_t)^2/\sigma_t}{\lambda+\sum_{s\in[t-1],z_s\in\cS^{\iota}}l(M_{z_t},\bar M_t,z_s)/\sigma_s}\\
&\le \frac{2^{-\iota}\sqrt{2c_0}2^{\iota/2}\log N/C}{\lambda+(j(z_i)-1)2^{-\iota-1}\cdot2^{\iota/2}\log N/C}\\
&\le \frac{\sqrt{8c_0}}{j(z_i)-1+\lambda2^{\iota/2+1}C/\log N}\\
&\le \frac{\sqrt{8c_0}}{j(z_i)-1+\lambda2^{\iota+1}},
\$
where the last inequality uses $C\ge2^{\iota/2}\log N$. By summing over all $z_t\in\cS^{\iota}$, we have
\#\label{eq:D_lambda_F_t_3}
\sum_{z_t\in\cS^{\iota}} (I_{\sigma}(\lambda,\cM,\cS_t))^2 &\le \sum_{j=1}^{N^{\iota}}\sum_{z_i\in\cS_j^{\iota}} \frac{\sqrt{8c_0}}{j-1 + \lambda2^{\iota+1}} + \sum_{z_i\in\cS_{N^{\iota}+1}^{\iota}} \frac{2}{N^{\iota}}\nonumber\\
&\le (\sqrt{8c_0}+2)\ED(\cM,2^{(-\iota-1)/2})\log N^{\iota}.
\#

Finally, by combining \eqref{eq:D_lambda_F_t_1}, \eqref{eq:D_lambda_F_t_2} and \eqref{eq:D_lambda_F_t_3}, we have
\$
&\sum_{t=1}^T (I_{\sigma}(\lambda,\cM,\cS_t))^2\\
&\quad= \sum_{\iota=0}^{\log_2(T/\lambda)} \sum_{z_t\in\cS^{\iota}} (I_{\sigma}(\lambda,\cM,\cS_t))^2\\
&\quad\le \sum_{\iota=0}^{\iota_0}(\sqrt{8c_0}+2)\ED(\cM,2^{(-\iota-1)/2})\log N^{\iota} + \sum_{\iota=\iota_0+1}^{\log_2(T/\lambda)-1}4\ED(\cM,2^{(-\iota-1)/2})\log N^{\iota} + 1\\
&\quad\le (\sqrt{8c_0}+3)\ED(\cM,\sqrt{\lambda/T})\log_2(T/\lambda)\log T,
\$
where the last inequality uses the monotonicity of the eluder dimension. Note that if $\iota_0=-1$, let the sum from $0$ to $-1$ be $0$. Eventually, we accomplish the proof due to the arbitrariness of $Z_1^T$.
\end{proof}

\subsection{Lemmas for Offline Setting}\label{ss:Lemmas for Offline Setting}

\begin{proof}[Proof of Lemma \ref{lm:off_Confidence Set}]
For simplicity, we assume that class $\cM$ has finite elements.
This proof is the same with the proof of Lemma \ref{lm:confidence_set} except for the formulation of the weights. For conciseness, we only present the difference here. Let
\[
E = \sum_{t=1}^T \tv(P_*^h(\cdot|z_t^h)\|P_{\bar M}^h(\cdot|z_t^h))^2/\sigma_t^h.
\]
Similar to \eqref{eq:ap_aah}, we can deduce that
\begin{align*}
&\sum_{t=1}^T\frac{1}{\sigma_t^h}\log\sqrt{\frac{\md P^h_{\bar M}(x^{h+1}|z_t^h)}{\md P^h_*(x^{h+1}|z_t^h)}}\notag\\
&\qquad\le \frac{19\log(|\cM|/\delta)\log^2 B}{3} -\frac{1}{4}\sum_{t=1}^T\frac{1}{\sigma_t^h} H\big(P^h_*(\cdot|z_t^h)\big\|P^h_{\bar{M}}(\cdot|z_t^h)\big)^2 + \sum_{t=1}^T\frac{c_t^h\tv\big(P^h_*(\cdot|z_t^h)\big\|P^h_{\bar M}(\cdot|z_t^h)\big)}{\sigma_t^h}\notag\\
&\qquad\le \frac{19\log(|\cM|/\delta)\log^2 B}{3} -\frac{1}{4}\sum_{s=1}^{t-1}\frac{1}{\sigma_s^h} \tv\big(P^h_*(\cdot|z_s^h)\big\|P^h_{\bar M_t}(\cdot|z_s^h)\big)^2 + 2\alpha C\sqrt{\lambda+E},
\end{align*}
where the last inequality uses Lemma \ref{lm:converge_weight}:
\begin{align*}
\sigma_t^h \ge & \frac{1}{2\alpha}\cdot\sup_{M,M'\in\cM}\frac{\tv(P_M^h(\cdot|z_t)\|P_{M'}^h(\cdot|z_t))/\alpha}{\sqrt{\lambda + \sum_{s=1}^T\tv(P_M^h(\cdot|z_s)\|P_{M'}^h(\cdot|z_s))^2/\sigma_s^2}}\\
\ge & \frac{1}{2\alpha}\cdot\frac{\tv(P_*^h(\cdot|z_t)\|P_{\bar M}^h(\cdot|z_t))/\alpha}{\sqrt{\lambda + \sum_{s=1}^T\tv(P_*^h(\cdot|z_s)\|P_{\bar M}^h(\cdot|z_s))^2/\sigma_s^2}}.
\end{align*}
Since $\bar M$ is the maximizer of the log-likelihood, we have
\[
E \le \frac{74\log(|\cM|/\delta)\log^2 B}{3}
+ 8\alpha C\sqrt{\lambda+E},
\]
which implies that
\[
E \le 2\beta^2.
\]
On the other hand, we get
\$
\sum_{t=1}^T \frac{1}{\sigma_t^h}\log P_*^h(x_t^{h+1}|z_t^h) &\ge \sum_{t=1}^T \frac{1}{\sigma_t^h}\log P_{\bar{M}}^h(x_t^{h+1}|z_t^h) - \frac{38\log(|\cM|/\delta)\log^2 B}{3} - 4\alpha C\sqrt{\lambda+E},
\$
which implies $M_*\in\hcM$.
\end{proof}

The proof adapts the analysis of Lemma 4.1 in \citet{ye2023corruptionb}.
\begin{proof}[Proof of Lemma \ref{lm:off Connections between Weighted and Unweighted Coefficient}]
For convenience, we use the short-hand notation 
\$
l^h(M,M',z)=\tv(P_M^h(\cdot|z)\|P_{M'}^h(\cdot|z)).
\$
Recall from Definition \ref{df:information coefficient} that
\[
\IC^{\sigma}(\lambda,\hcM,\cD) = T\cdot\max_{h\in[H]} \E_{\pi_*}\bigg[\sup_{M,M'\in\hcM} \frac{l^h(M,M',z^h)^2/\sigma^h(z^h)}{\lambda + \sum_{t=1}^Tl^h(M,M',z_t^h)^2/\sigma_t^h} \bigg| x^1=x\bigg],
\]
where we define
\[
\sigma^h(z^h) = \max\bigg\{1, \sup_{M,M'\in\hcM}\frac{l^h(M,M',z^h)/\alpha}{\sqrt{\lambda + \sum_{t=1}^Tl^h(M,M',z_t^h)^2/\sigma_t^h}}\bigg\}.
\]
Besides, Assumption \ref{as:offline_upper} states that for any $h\in[H]$, and two distinct $M,M'\in\cM$,
\$
\frac{1}{T}\sum_{t=1}^Tl^h(M,M',z_t^h)^2 \ge \Cov(\cM,\cD)\rho^h(M,M')^2,
\$
where $\rho^h(M,M')=\sup_{z}\tv(P^h_{M}(\cdot|z)\|P^h_{M'}(\cdot|z))$. 

Let $M_{z^h},M_{z^h}'$ be the models that maximize
\$
\frac{l^h(M,M',z^h)^2 / \sigma^h(z^h)}{\lambda + \sum_{t=1}^T l^h(M,M',z_t^h)^2/\sigma_t^h}.
\$
We use the notation
\$
\psi(z^h) = \frac{l^h(M_{z^h},M'_{z^h},z^h)^2 / \sigma^h(z^h)}{\lambda + \sum_{t=1}^T l^h(M_{z^h},M'_{z^h},z^h)^2/\sigma_t^h}.
\$
Since 
\$
\sigma^h(z^h) \ge & \sup_{M,M'\in\hcM}\frac{l^h(M,M',z^h)/\alpha}{\sqrt{\lambda + \sum_{t=1}^Tl^h(M,M',z_t^h)^2/\sigma_t^h}}\\
= & \frac{l^h(M_{z^h},M'_{z^h},z^h)/\alpha}{\sqrt{\lambda + \sum_{t=1}^T l^h(M_{z^h},M'_{z^h},z_t^h)^2/\sigma_t^h}},
\$
we deduce that
\#\label{eqap:sigma(zh)2_bound}
(\sigma^h(z^h))^{1/2} \ge & \frac{1}{\alpha}\cdot\frac{l^h(M_{z^h},M'_{z^h},z^h)/(\sigma^h(z^h))^{1/2}}{\sqrt{\lambda + \sum_{t=1}^T l^h(M_{z^h},M'_{z^h},z_t^h)^2/\sigma_t^h}}\notag\\
= & \frac{\sqrt{\psi(z^h)}}{\alpha}.
\#

Then, we will derive a uniform upper bound for $\sigma_t^h$ for $t\in[T]$. For all $t\in[T]$, we have from Lemma \ref{lm:converge_weight} that
\$
\sigma_t^h \le & \max\bigg\{1, \sup_{M,M'\in\cM}\frac{l^h(M,M',z_t^h)/\alpha}{\sqrt{\lambda + \sum_{s=1}^T l^h(M,M',z_t^h)^2 / \sigma_s^h}}\bigg\}\\
\le & \max\bigg\{1, \sup_{M,M'\in\cM}\frac{l^h(M,M',z_t^h)/\alpha}{\sqrt{\lambda + \sum_{s=1}^T l^h(M,M',z_t^h)^2 / \max_s\sigma_s^h}}\bigg\}\\
\le & \max_s\sqrt{\sigma_s^h}\cdot\max\bigg\{1, \sup_{M,M'\in\cM}\frac{l^h(M,M',z_t^h)/\alpha}{\sqrt{\lambda + \sum_{s=1}^T l^h(M,M',z_t^h)^2}}\bigg\}\\
\le & \max_s\sqrt{\sigma_s^h}\cdot\max\bigg\{1, \sup_{M,M'\in\cM}\frac{l^h(M,M',z_t^h)/\alpha}{\sqrt{\lambda + \sum_{s=1}^T l^h(M,M',z_t^h)^2}}\bigg\}.
\$
By using Assumption \ref{as:offline_upper}, we further have
\$
\sigma_t^h \le &
\max_s\sqrt{\sigma_s^h}\cdot\max\bigg\{1, \sup_{M,M'\in\cM}\frac{l^h(M,M',z_t^h)/\alpha}{\sqrt{\lambda + T\Cov(\cM,\cD)\rho^h(M,M')^2}}\bigg\}\\
\le & \max_s\sqrt{\sigma_s^h}\cdot\max\bigg\{1, \frac{1}{\alpha\sqrt{T\Cov(\cM,\cD)}}\bigg\},
\$
which implies that
\#\label{eqap:sigmath_bound}
\max_t\sigma_t^h \le \max\Big\{1, \frac{1}{\alpha^2T\Cov(\cM,\cD)}\Big\}.
\#

There are two situations. If $\alpha^2T\Cov(\cM,\cD) \ge 1$, we have from \eqref{eqap:sigmath_bound} that $\sigma_t^h=1$ for all $t\in[T]$. Hence, since $\sigma^h(z^h)\ge 1$ it follows that
\$
\IC^{\sigma}(\lambda,\hcM,\cD) \le & T\cdot\max_{h\in[H]} \E_{\pi_*}\bigg[\sup_{M,M'\in\hcM} \frac{l^h(M,M',z^h)^2}{\lambda + \sum_{t=1}^Tl^h(M,M',z_t^h)^2} \bigg]\\
\le & T\cdot\max_{h\in[H]}\E_{\pi_*}\bigg[\sup_{M,M'\in\hcM} \frac{l^h(M,M',z^h)^2}{\lambda + T\Cov(\cM,\cD)\rho(M,M')^2} \bigg]\\
\le & T\cdot\frac{1}{T\Cov(\cM,\cD)} = \frac{1}{\Cov(\cM,\cD)}.
\$

If $\alpha^2T\Cov(\cM,\cD) \ge 1$, we combine \eqref{eqap:sigma(zh)2_bound} and \eqref{eqap:sigmath_bound} to get
\$
\psi(z^h) \le & \frac{l^h(M_{z^h},M'_{z^h},z^h)^2 \cdot \alpha^2 / \psi(z^h)}{\lambda + \sum_{t=1}^T l^h(M_{z^h},M'_{z^h},z^h)^2\cdot \alpha^2T\Cov(\cM,\cD)}\\
\le & \frac{1}{\psi(z^h)} \cdot \frac{l^h(M_{z^h},M'_{z^h},z^h)^2}{T\Cov(\cM,\cD)\cdot T\Cov(\cM,\cD)\rho(M_{z^h},M'_{z^h})^2}\\
\le & \frac{1}{\psi(z^h)} \cdot \frac{1}{(T\Cov(\cM,\cD))^2},
\$
which implies that
\$
\psi(z^h) \le \frac{1}{T\Cov(\cM,\cD)}.
\$

Therefore, we obtain
\$
\IC^{\sigma}(\lambda,\hcM,\cD) \le \frac{1}{\Cov(\cM,\cD)}.
\$
\end{proof}

\section{Auxiliary Results}
\subsection{Bounding TV-Eluder Dimension for Tabular MDPs}
\begin{theorem}\label{th:TV-Eluder Dimension for Tabular MDPs}
Consider a family of tabular MDPs with transition $P_M^h:\cX\times\cA\rightarrow[0,1]$ for any $M\in\cM,h\in[H]$. Then, we have
\[
\ED(\cM,\epsilon) \le 48SA\log(1+8SA/\epsilon^2).
\]
\end{theorem}
\begin{proof}
For any $i\in[n],z\in\cX\times\cA$, denote $\tv(P_{M}^h(\cdot|z)\|P_{M'}^h(\cdot|z))$ by $l_i(z)$.
Let $\{(z_i,l_i)\}_{i=1}^n$ be any sequence such that for any $i\in[n]$, $(z_i,l_i)$ is $\epsilon$-independent of its predecessors. We can formulate $z_i$ and $l_i$ as $SA$-dimensional vectors, respectively. Let $\Sigma_i=\sum_{s=1}^{i-1}z_iz_i^{\top} + \lambda I$ for any $i\in[n]$. Then, we get
\#\label{eqap:aao}
l_i(z_i) = & z_i^{\top}l_i \le \|z_i\|_{\Sigma_i^{-1}} \|l_i\|_{\Sigma_i}\notag\\
\le & \|z_i\|_{\Sigma_i^{-1}}\sqrt{l_i^{\top}\Big(\sum_{s=1}^{i-1}z_iz_i^{\top} + \lambda I\Big)l_i}\notag\\
= & \|z_i\|_{\Sigma_i^{-1}}\sqrt{\sum_{s=1}^{i-1}(l_i^{\top}z_i)^2 + \lambda SA}.
\#
On the one hand, we have
\$
\sum_{i=1}^n l_i(z_i) > n\epsilon.
\$
On the other hand, we get
\$
\sum_{i=1}^n l_i(z_i) \le & \sum_{i=1}^n \|z_i\|_{\Sigma_i^{-1}}\sqrt{\sum_{s=1}^{i-1}(l_i^{\top}z_i)^2 + \lambda SA}\\
\le & \sum_{i=1}^n \|z_i\|_{\Sigma_i^{-1}}\sqrt{\epsilon^2 + \lambda SA}\\
\le & (\epsilon + \sqrt{\lambda SA}) \sum_{i=1}^n \|z_i\|_{\Sigma_i^{-1}}\\
\le & 2\epsilon\sqrt{T\sum_{i=1}^n \|z_i\|_{\Sigma_i^{-1}}^2} \le 2\epsilon\sqrt{2nSA\log(1+n/\epsilon^2)},
\$
where the first inequality uses \eqref{eqap:aao},  the second inequality holds due to the definition of $\epsilon$-independence, and the last inequality applies elliptical potential lemma (Lemma \ref{lm:potential}) and takes $\lambda=\epsilon^2/SA$. It follows that
\$
n\epsilon \le 2\epsilon\sqrt{2nSA\log(1+n/\epsilon^2)},
\$
which implies that
\$
n \le 8SA\log(1+n/\epsilon^2).
\$
According to Lemma G.5 of \citet{wang2023benefits}, we obtain that
\$
n \le 48SA\log(1+8SA/\epsilon^2).
\$
\end{proof}

\subsection{Bounding TV-Eluder Dimension for Linear MDPs}
\begin{theorem}\label{th:TV-Eluder Dimension for Linear MDPs}
Consider a family of linear MDPs, and there exist maps $\nu^h:\cM\times\cX\rightarrow\rR^d$ and $\phi^h:\cX\times\cA\rightarrow\rR^d$ such that transition $P^h_M(x^{h+1}|z^h)=\nu^h(M,x^{h+1})^{\top}\phi^h(z^h)$, where $\|\phi(z)\|_2\le1$ for any $z\in\cX\times\cA$. Then, we have
$$
\ED(\cM,\epsilon) \le 48d\log(1+8d/\epsilon^2)
$$
\end{theorem}
\begin{proof}
For any $i\in[n],z\in\cX\times\cA$, denote $\tv(P_{M_i}^h(\cdot|z)\|P_{M_i'}^h(\cdot|z))$ by $l_i(z)$.
Let $\{(z_i,l_i)\}_{i=1}^n$ be any sequence such that for any $i\in[n]$, $(z_i,l_i)$ is $\epsilon$-independent of its predecessors. Let $\Sigma_i=\sum_{s=1}^{i-1}\phi^h(z_s)\phi^h(z_s)^{\top} + \lambda I$.
For any $M,M'\in\cM$, we have
\$
\tv(P_{M}^h(\cdot|z_i)\|P_{M'}^h(\cdot|z_i)) = & \sup_{\cX_0} \Big|\int_{\cX_0} \big(\nu^h(M,x^{h+1}) - \nu^h(M',x^{h+1})\big)^{\top}\phi^h(z_i) \md x^{h+1}\Big|\\
\le & \|\phi^h(z_i)\|_{\Sigma_i^{-1}} \sup_{\cX_0}\Big\|\int_{\cX_0} \big(\nu^h(M,x^{h+1}) - \nu^h(M',x^{h+1})\big) \md x^{h+1}\Big\|_{\Sigma_i}.
\$
Since 
\$
& \sup_{\cX_0}\Big\|\int_{\cX_0} \big(\nu^h(M,x^{h+1}) - \nu^h(M',x^{h+1})\big) \md x^{h+1}\Big\|_{\Sigma_i}^2\\
&\qquad \le \sup_{\cX_0} \sum_{s=1}^{i-1} \Big(\int_{\cX_0}\big(\nu^h(M,x^{h+1}) - \nu^h(M',x^{h+1})\big)^{\top}\phi^h(z_s)\Big)^2 + \lambda d\\
&\qquad\le \sum_{s=1}^{i-1} \Big(\sup_{\cX_0} \int_{\cX_0}\big(\nu^h(M,x^{h+1}) - \nu^h(M',x^{h+1})\big)^{\top}\phi^h(z_s)\Big)^2 + \lambda d\\
&\qquad\le \sum_{s=1}^{i-1} \tv(P^h_{M}(\cdot|z_s)\|P^h_{M'}(\cdot|z_s))^2 + \lambda d.
\$
It follows that
\#\label{eqap:aap}
\tv(P_{M}^h(\cdot|z_i)\|P_{M'}^h(\cdot|z_i)) \le \|\phi^h(z_i)\|_{\Sigma_i^{-1}}\sqrt{\sum_{s=1}^{i-1} \tv(P^h_{M}(\cdot|z_s)\|P^h_{M'}(\cdot|z_s))^2 + \lambda d}.
\#
On the one hand, we have
\$
\sum_{i=1}^n l_i(z_i) > n\epsilon.
\$
On the other hand, we derive from \eqref{eqap:aap} that
\$
\sum_{i=1}^n l_i(z_i) \le & \sum_{i=1}^n\|\phi^h(z_i)\|_{\Sigma_i^{-1}}\sqrt{\sum_{s=1}^{i-1} l_i(z_s)^2 + \lambda d}\\
\le & \sqrt{\epsilon^2 + \lambda d}\sum_{i=1}^n\|\phi^h(z_i)\|_{\Sigma_i^{-1}}\\
\le & (\epsilon + \sqrt{\lambda d})\sqrt{n\sum_{i=1}^n\|\phi^h(z_i)\|_{\Sigma_i^{-1}}^2}\\
\le & 2\epsilon\sqrt{2nd\log(1+n/\epsilon^2)},
\$
where the last inequality holds by invoking Lemma \ref{lm:potential} and taking $\lambda=\epsilon^2/d$. Therefore, we obtain
\$
n\epsilon \le 2\epsilon\sqrt{2nd\log(1+n/\epsilon^2)}.
\$
Then, by applying Lemma G.5 of \citet{wang2023benefits}, we complete the proof.
\end{proof}

\begin{example}[Information Ratio for Linear Model]\label{eg:Information Ratio for Linear Model}
If the transition can be embedded as $P^h_M(x^{h+1}|z^h)=\nu^h(M,x^{h+1})^{\top}\phi^h(z^h)$, the IR defined in \eqref{eq:tv_information_ratio}
\$
    I^h(\lambda,\cM,\cS_t^h) \le \min\Big\{1,\|\phi^h(z_t^h)\|_{(\Sigma_t^h)^{-1}}\Big\}.
\$
\end{example}
\begin{proof}
Let 
\$
\Sigma_t^h = \sum_{s=1}^{t-1}\phi^h(z_s^h)\phi^h(z_s^h)^{\top}.
\$
We deduce from the definition of TV-norm and the Cauchy-Schwartz inequality that
\$
\tv(P_{M}^h(\cdot|z_t^h)\|P_{\bar M_t}^h(\cdot|z_t^h)) =& \sup_{\cX'}\Big|\phi^h(z_t^h)^{\top}\int_{\cX'}(\nu^h(M,x^{h+1}) - \nu^h(\bar M_t,x^{h+1}))\md x^{h+1}\Big|\\
\le & \|\phi^h(z_t^h)\|_{(\Sigma_t^h)^{-1}} \cdot \sup_{\cX'}\Big\| \int_{\cX'} (\nu^h(M,x^{h+1}) - \nu^h(\bar M_t,x^{h+1}))\md x^{h+1}\Big\|_{\Sigma_t^h}.
\$
We define 
\$
\cX_t = \argmax_{\cX'\subseteq\cX}\Big\| \int_{\cX'} (\nu^h(M,x^{h+1}) - \nu^h(\bar M_t,x^{h+1}))\md x^{h+1}\Big\|_{\Sigma_t^h}.
\$
Additionally, we get
\$
&\lambda+\sum_{s=1}^{t-1}\tv(P_{M}^h(\cdot|z_s^h)\|P_{\bar M_t}^h(\cdot|z_s^h))^2 \\
&\qquad = \lambda + \frac{1}{4}\sum_{s=1}^{t-1} \sup_{\cX'}\Big(\phi^h(z_s^h)^{\top}\int_{\cX'}(\nu^h(M,x^{h+1}) - \nu^h(\bar M_t,x^{h+1}))\md x^{h+1}\Big)^2\\
&\qquad \ge \lambda + \frac{1}{4}\sum_{s=1}^{t-1} \Big(\phi^h(z_s^h)^{\top}\int_{\cX_t}(\nu^h(M,x^{h+1}) - \nu^h(\bar M_t,x^{h+1}))\md x^{h+1}\Big)^2\\
&\qquad = \lambda + \Big(\int_{\cX_t}(\nu^h(M,x^{h+1}) - \nu^h(\bar M_t,x^{h+1}))\md x^{h+1}\Big)^{\top} \sum_{s=1}^{t-1}\phi^h(z_s^h)\phi^h(z_s^h)^{\top}\\
&\qquad\qquad \Big(\int_{\cX_t}(\nu^h(M,x^{h+1}) - \nu^h(\bar M_t,x^{h+1}))\md x^{h+1}\Big)\\
&\qquad \ge \Big\|\int_{\cX_t}(\nu^h(M,x^{h+1}) - \nu^h(\bar M_t,x^{h+1}))\md x^{h+1}\Big\|_{\Sigma_t^h}^2.
\$
Hence, we have
\$
&\frac{\tv(P_{M}^h(\cdot|z_t^h)\|P_{\bar M_t}^h(\cdot|z_t^h))}{\sqrt{\lambda+\sum_{s=1}^{t-1}\tv(P_{M}^h(\cdot|z_s^h)\|P_{\bar M_t}^h(\cdot|z_s^h))^2}}\\
&\qquad \le \frac{\|\phi^h(z_t^h)\|_{(\Sigma_t^h)^{-1}} \cdot \Big\| \int_{\cX_t} (\nu^h(M,x^{h+1}) - \nu^h(\bar M_t,x^{h+1}))\md x^{h+1}\Big\|_{\Sigma_t^h}}{\Big\|\int_{\cX_t}(\nu^h(M,x^{h+1}) - \nu^h(\bar M_t,x^{h+1}))\md x^{h+1}\Big\|_{\Sigma_t^h}}\\
&\qquad \le \|\phi^h(z_t^h)\|_{(\Sigma_t^h)^{-1}},
\$
which conludes the proof.
\end{proof}

Now, we use the linear MDP to illustrate the condition in Assumption \ref{as:offline_upper}. In the following lemma, we demonstrate that the condition holds as long as the learner has excess to a well-explored dataset \eqref{eq:minimum eigenvalue condition}, which is a wildly-adopted assumption in the literature of offline linear MDPs \citep{duan2020minimax,wang2020statistical,zhong2022pessimistic}.
\begin{lemma}\label{lm:minimum eigenvalue condition and condition for W and UW connections}
In the linear setting where the transition model $M$ can be embedded into a $d$-dimensional vector space: $\cM^h=\{\langle \nu^h(M,x^{h+1}),\phi^h(\cdot) \rangle : \cX\rightarrow\rR\}$ and $\|\phi(z)\|\le 1$, if we assume that the data empirical covariance satisfies
the following minimum eigenvalue condition: there exists an absolute constant $\bar{c}>0$ such that
\#\label{eq:minimum eigenvalue condition}
\sigma_{\min}\Big(T^{-1}\sum_{t=1}^T \phi(z_t^h)\phi(z_t^h)^\top\Big)=\frac{\bar{c}}{d},
\#
then, Assumption \ref{as:offline_upper}: 
for any two distinct $M,M'\in\cM$,
\$
\min_{h\in[H]}\frac{1}{T}\sum_{t=1}^T\tv(P_{M}^h(\cdot|z_t^h)\|P_{M'}^h(\cdot|z_t^h))^2 \ge \Cov(\cM,\cD) \rho(M,M')^2
\$
with $\Cov(\cM,\cD)=\bar{c}/(2d)$ will holds with probability at least $1-\delta$.
\end{lemma}
\begin{proof}
By using the linear model and the definition of TV-norm, we have for some $M,M'\in\cM$,
\$
&\frac{1}{T}\sum_{t=1}^T\tv(P_{M}^h(\cdot|z_t^h)\|P_{M'}^h(\cdot|z_t^h))^2\notag\\
&\qquad = \frac{1}{T}\sum_{t=1}^T\sup_{\cX'}\Big(\phi^h(z_t^h)^{\top}\int_{\cX'}(\nu^h(M,x^{h+1}) - \nu^h(M',x^{h+1}))\md x^{h+1}\Big)^2\notag\\
&\qquad \ge \sup_{\cX'}\frac{1}{T}\sum_{t=1}^T\Big(\phi^h(z_t^h)^{\top}\int_{\cX'}(\nu^h(M,x^{h+1}) - \nu^h(M',x^{h+1}))\md x^{h+1}\Big)^2\notag\\
&\qquad = \sup_{\cX'} \Big(\int_{\cX'}(\nu^h(M,x^{h+1}) - \nu^h(M',x^{h+1}))\md x^{h+1}\Big)^{\top} \Lambda_T^h \Big(\int_{\cX'}(\nu^h(M,x^{h+1}) - \nu^h(M',x^{h+1}))\md x^{h+1}\Big)\notag\\
&\qquad \ge \frac{\bar{c}}{d}\sup_{\cX'}\Big\|\int_{\cX'}(\nu^h(M,x^{h+1}) - \nu^h(M',x^{h+1}))\md x^{h+1}\Big\|^2,
\$
where we define $\Lambda_T^h=T^{-1}\sum_{s=1}^{t-1}\phi^h(z_s^h)\phi^h(z_s^h)^{\top}$, and the last inequality uses $\lambda_{\min}(\Lambda_T^h)=\bar{c}/d$. We use the short-hand notation
\$
\xi = \int_{\cX_0}(\nu^h(M,x^{h+1}) - \nu^h(M',x^{h+1}))\md x^{h+1},
\$
where $\cX_0$ is the maximizer of \eqref{eqap:aaq}.
Then, it follows that
\#\label{eqap:aaq}
\frac{1}{T}\sum_{t=1}^T\tv(P_{M}^h(\cdot|z_t^h)\|P_{M'}^h(\cdot|z_t^h))^2 \ge \frac{\bar{c}}{d}\|\xi\|^2.
\#

Additionally, we also have
\$
\rho(M,M') = & \sup_{z}\tv(P_M^h(\cdot|z)\|P_{M'}^h(\cdot|z))\\
\le & \sup_{z}\sup_{\cX'}\Big(\phi^h(z)^{\top}\int_{\cX'}(\nu^h(M,x^{h+1}) - \nu^h(M',x^{h+1}))\md x^{h+1}\Big)^2\\
\le & \sup_{\cX'}\Big\|\int_{\cX'}(\nu^h(M,x^{h+1}) - \nu^h(M',x^{h+1}))\md x^{h+1}\Big\|^2\\
= & \|\xi\|^2.
\$
Therefore, we complete the proof.
\end{proof}

\section{Technical Lemmas}

\begin{lemma}[Azuma–Hoeffding inequality, \citealt{cesa2006prediction}]\label{lemma:azuma}
Let $\{x_i\}_{i=1}^n$ be a martingale difference sequence with respect to a filtration $\{\cG_{i}\}$ satisfying $|x_i| \leq M$ for some constant $M$, $x_i$ is $\cG_{i+1}$-measurable, $\E[x_i|\cG_i] = 0$. Then for any $0<\delta<1$, with probability at least $1-\delta$, we have 
\begin{align}
    \sum_{i=1}^n x_i\leq M\sqrt{2n \log (1/\delta)}.\notag
\end{align} 
\end{lemma}


\begin{lemma}[Theorem 13.6 of \citet{zhang2023mathematical}]\label{lm:freedman}
Consider a sequence of random functions $C_1(\cS_1),\ldots,C_t(\cS_t)$. Assume that $\xi_i\le\E_{Z_i^{(y)}}\xi_i + b$ for some constant $b > 0$. Then for any $\lambda\in(0,3/b)$, with probability at least $1-\delta$:
\$
\sum_{i=1}^n\xi_i\le \sum_{i=1}^n\E_{Z_i^{(y)}}\xi_i + \frac{\lambda\sum_{i=1}^n\Var_{Z_i^{(y)}}(\xi_i)}{2(1-\lambda b/3)} + \frac{\log(1/\delta)}{\lambda}.
\$
\end{lemma}

\begin{lemma}\label{lm:divergence inequality}
    Let $\rho=\sup_z\log(p(z)/q(z))$. Then,
    \$
        \int \md P(z)\log^2\Big(\frac{\md P(z)}{\md Q(z)}\Big) \le 2(\rho+1)\kl(P\|Q).
    \$
\end{lemma}
\begin{proof}
It is obvious that $\rho>0$.
Now let $f_{\kl}(t)=t\log t-t+1$ and $f(t)=t\log^2 t$. Define
\$
    \kappa = \sup_{0\le t\le \exp(\rho)} \frac{f(t)}{f_{\kl}(t)}.
\$
By using some algebra, we know that $f(t)/f_{\kl}(t)$ is an increasing function of $t\in[0,+\infty)$, which implies that when $\rho\in(0,+\infty)$
\$
    \kappa \le \frac{f(\exp(\rho))}{f_{\kl}(\exp(\rho))} = \frac{\rho^2\exp(\rho)}{\rho\exp(\rho)-\exp(\rho)+1} \le 2+2\rho.
\$
Therefore, we have
\$
    \int \md P(z)\log^2\Big(\frac{\md P(z)}{\md Q(z)}\Big) = \E_{z\sim Q}f\Big(\frac{p(z)}{q(z)}\Big) \le \kappa\E_{z\sim Q}f_{\kl}\Big(\frac{p(z)}{q(z)}\Big) = \kappa\kl(P\|Q),
\$
which indicates the desired bounds.
\end{proof}

\begin{lemma}[Proposition B.11 of \citet{zhang2023mathematical}]\label{lm:kl and Hellinger}
Let $\rho=\sup_z\log(p(z)/q(z))$. Then
\$
    H(P\|Q)^2 \le \kl(P\|Q) \le (3+\rho)H(P\|Q)^2.
\$
\end{lemma}

\begin{lemma}[Elliptical Potential Lemma \citep{dani2008stochastic,rusmevichientong2010linearly,abbasi2011improved}] \label{lm:potential}
    Let $\{x_i\}_{i \in [T]}$ be a sequence of vectors in $\rR^d$ with $\norm{x_i}_2 \leq L < \infty$ for all $t \in [T]$. Let $\Lambda_0$ be a positive-definite matrix and $\Lambda_t = \Lambda_0 + \sum_{i = 1}^{t} x_i x_i^\top$. It holds that
    $$
    \log\Big(\frac{\det(\Lambda_t)}{\Lambda_0}\Big) \le \sum_{i=1}^T \|x_i\|^2_{\Lambda_{i-1}^{-1}}.
    $$
    Further, if $\|x_i\|_2\le L$ for all $i\in[T]$, then we have
    $$
    \sum_{i=1}^T\min\{1, \|x_i\|^2_{\Lambda_{i-1}^{-1}}\} \le 2\log\Big(\frac{\det(\Lambda_t)}{\Lambda_0}\Big) \le 2d\log\Big(\frac{\mathrm{trace}(\Lambda_0) + nL^2}{d\det(\Lambda_0)^{1/d}}\Big).
    $$
    Finally, if $\lambda_{\min}(\Lambda_0) \ge \max(1,L^2)$, 
    $$
    \sum_{i=1}^T\|x_i\|^2_{\Lambda_{i-1}^{-1}} \le 2\log\Big(\frac{\det(\Lambda_t)}{\Lambda_0}\Big).
    $$
\end{lemma}

\end{document}